\documentclass{article}

\usepackage{microtype}
\usepackage{graphicx}
\usepackage{subfigure}
\usepackage{booktabs} 
\usepackage{amsmath}
\usepackage{amsfonts}
\usepackage{amsthm}

\usepackage{hyperref}

\usepackage[accepted]{icml2019}

\icmltitlerunning{Discovering Context Effects from Raw Choice Data}

\begin{document}
\setlength{\abovedisplayskip}{5pt}
\setlength{\belowdisplayskip}{5pt}

\newcommand{\xhdr}[1]{\paragraph*{\bf #1.}}

\newcommand\defeq{:=}
\newcommand\allsubsets{\mathcal{C}}
\newcommand\uniqueD{\mathcal{C}_{\mathcal{D}}}
\newcommand\E{\mathbb{E}}
\newcommand\Var{\text{Var}}
\newcommand\Cov{\text{Cov}}
\newcommand\abs{\text{abs}}
\newcommand\nper{{n_{\text{per}}}}
\newcommand\eff{\textbf{f}}
\newcommand\exx{\mathcal{X}}
\newtheorem{defn}{Definition}
\newtheorem{thm}{Theorem}
\newtheorem{cor}{Corrolary}
\newtheorem{ass}{Assumption}
\newtheorem{question}{Question}

\newtheorem{lemma}{Lemma}
\newtheorem{fact}{Fact}

\newcommand*{\alex}{\textcolor{red}}
\newcommand*{\as}{\textcolor{blue}}
\newcommand{\jucom}[1]{\textcolor{red}{[#1]}}
\newcommand{\ascom}[1]{\textcolor{blue}{[#1]}}

\twocolumn[
\icmltitle{Discovering Context Effects from Raw Choice Data}



\icmlsetsymbol{equal}{*}

\begin{icmlauthorlist}
\icmlauthor{Arjun Seshadri}{a}
\icmlauthor{Alexander Peysakhovich}{b}
\icmlauthor{Johan Ugander}{a}
\end{icmlauthorlist}

\icmlaffiliation{a}{Stanford University, Stanford, CA}
\icmlaffiliation{b}{Facebook Artificial Intelligence Research, New York, NY}

\icmlcorrespondingauthor{Arjun Seshadri}{aseshadr@stanford.edu}
\icmlcorrespondingauthor{Alexander Peysakhovich}{alexpeys@fb.com}
\icmlcorrespondingauthor{Johan Ugander}{jugander@stanford.edu}

\icmlkeywords{}

\vskip 0.3in
]



\printAffiliationsAndNotice{}  

%

\begin{abstract}
Many applications in preference learning assume that decisions come from the maximization of a stable utility function. Yet a large experimental literature shows that individual choices and judgements can be affected by ``irrelevant'' aspects of the context in which they are made. An important class of such contexts is the composition of the choice set. In this work, our goal is to discover such choice set effects from raw choice data. We introduce an extension of the Multinomial Logit (MNL) model, called the context dependent random utility model (CDM), which allows for a particular class of choice set effects. We show that the CDM can be thought of as a second-order approximation to a general choice system, can be inferred optimally using maximum likelihood and, importantly, is easily interpretable. We apply the CDM to both real and simulated choice data to perform principled exploratory analyses for the presence of choice set effects.
\end{abstract}

\section{Introduction}

Modeling individual choice is an important component of recommender systems \cite{resnick1997recommender}, search engine ranking  \cite{schapire1998learning}, analysis of auctions \cite{athey2001information}, marketing \cite{allenby1998marketing}, and demand modeling in diverse domains \cite{berry1995automobile,bruch2016extracting}. The workhorse models used either implicitly or explicitly in these disparate literatures are random utility models (RUMs) \cite{manski1977structure}, which assume that individuals have a numeric utility for each item and that they make choices that maximize noisy observations of these utilities \cite{luce1959individual,mcfadden1980econometric,kreps1988notes}. 

The most well known and widely used RUM is the conditional multinomial logit (MNL), also called the Luce model, which is the unique RUM that satisfies the axiom known as the {\it independence of irrelevant alternatives (IIA)} \cite{luce1959individual}. Informally, this axiom states that adding an item to a choice set does not change the relative probabilities of choosing the other items. This assumption is very strong, but it allows analysts to build powerful and interpretable models. However, if one assumes IIA but it is not actually true, predictions for out-of-sample choices could be very wrong. Thus, it is important for analysts to discover whether IIA is approximately true in a given dataset.

At the same time, there is a large amount of experimental evidence showing significant deviations from rational choice across many domains.  In particular, the value assigned to an item can strongly depend on the ``irrelevant'' elements of the context of the choice \cite{tversky1972elimination,tversky1993context}. Attempts to model these context effects in a domain-free manner fall short of being practically valuable, either due to large parameter requirements, inferential intractability, or both \cite{park2013theoretical}. 

Our contribution addresses both issues. We consider the IIA-satisfying MNL model and make small modifications to subsume a class of IIA violations that we believe are important in practice, while retaining parametric and inferential efficiency. We refer to this model as the context dependent random utility model (CDM). The CDM can be thought of as a ``second order'' approximation of a general choice system (the MNL model, meanwhile, corresponds to a ``first order'' approximation). Because the CDM nests MNL, it can fit data that does satisfy IIA just as well, and, importantly, can be used to construct a nested-model hypothesis test for whether a particular dataset is consistent with IIA.

The key assumption of the CDM is that IIA violations come from pairwise interactions between items in the choice set and that larger choice set effects can be approximated additively using all pairwise effects. This assumption means that the CDM has many fewer parameters than a general choice system. We can further reduce the CDM's data dependence by assuming that these underlying effects can be well modeled by latent vectors of a smaller dimensionality than the number of items, resulting in what we call the low-rank CDM. The low-rank CDM can be useful in applications where the number of items is relatively large and where seeing all possible comparisons may be extremely costly.

For a theoretical contribution, we furnish formal results for conditions under which the parameters of a CDM can or can not be identified from data. In situations where identifiability is not achieved, we advocate for additive $\ell_2$-regularization of the log-likelihood to select the minimum norm solution. We also provide finite sample convergence guarantees for the expected squared $\ell_2$ error of the estimate as a function of comparison structure. 

For an applied contribution, we first test the CDM in synthetic data and show that a nested model likelihood ratio test---between the CDM and MNL models---has good finite sample properties. When IIA holds, a $p < .05$ hypothesis test rejects the null slightly less than $5\%$ of the time. When IIA does not hold the null hypothesis is overwhelmingly rejected even in medium size data-sets. By contrast we see that using a nested model test based on the nested structure of a general choice system and MNL model gives a test that wildly over-rejects the null, even when IIA is true.

We apply the CDM to several real-world datasets. First, we show that we can strongly reject IIA in the popular SFWork and SFShop choice datasets. Second, we consider using the CDM to model choices in the task of \citet{heikinheimo2013crowd}, where individuals are presented with triplets of items and asked which item is least like the other two. Here the CDM can capture the underlying choice structure quite well while IIA is an extremely unreasonable assumption.

\subsection{Related Work}

The CDM vaguely resembles the continuous bag of words (CBOW) neural network architecture popularized by word2vec \cite{mikolov2013distributed}, with two key differences. First, while the CBOW model tries to predict the appearance of a word where candidates are any word in the vocabulary, the CDM models choices from arbitrary subsets. Second, although in principle the word2vec model and its extensions train two embeddings per word (one as target and one as context), these embeddings are typically averaged together at the end of training to obtain a single embedding per word. However, it has been shown that keeping these two embeddings separate does allow one to capture ancillary information not captured by the single embedding \cite{rudolph2016exponential} and the two embeddings can be used to model complements and substitutes in supermarket shopping data \cite{ruiz2017shopper}.

Utility models with a contextual component are widely used to analyze intertemporal choice (discount functions) \cite{muraven2000self, fudenberg2012timing}, choice under uncertainty \cite{bordalo2012salience, fox1995ambiguity}, and choices about cooperation \cite{list2007interpretation, liberman2004name, peysakhovich2015habits}. 
They are also workhorses in modeling consumer behavior in applied settings. 
Online recommender systems \cite{resnick1997recommender}, which model user-item interactions as inner products of low rank vectors, can be seen as employing a utility function that is an inner product between item attributes and user weights.

The CDM and low-rank CDM generalize a number of prominent choice models in a unified framework. In a later section, after we introduce the basic mathematical notation, we present connections to the work of \citet{tversky1993context}, \citet{batsell1985new}, and \citet{chen2016modeling,chen2016predicting}. Importantly, this means our convergence and identifiability results carry over to these other models, which all previously lacked such results.

\section{Modeling Choice Systems}
\label{sec:choicesystem}

Let $\mathcal{X}$ be a finite set of $n$ alternatives that we hold fixed and let $\allsubsets = \{C : C \subseteq \mathcal{X}, |C| \geq 2\}$ be the set of all subsets of $\mathcal X$ of size greater than or equal to two. Throughout this work, we assume there is a single individual that is presented with choice sets and chooses a single item from each choice set. In this setting, the fundamental object of study is a \textit{choice system}, a collection of probability distributions for every $C \in \allsubsets$, describing the probability that an item $x$ is chosen from $C, \forall x \in C$. We denote each such probability by $P(x \mid C)$. In general, a choice system can model arbitrary preferences on arbitrary subsets with no further restrictions.

The most commonly assumed restriction on choice systems is that they satisfy the independence of irrelevant alternatives (IIA), which can be stated as follows.
\begin{ass}
A choice system on $\mathcal X$ satisfies the independence of irrelevant alternatives (IIA) if for any $x, y \in \mathcal{X}$ and choice sets $A, B  \subseteq \mathcal X$ with $x, y \in A, B$ we have 
\[\frac{P(x \mid A)}{P (y \mid A)} = \frac{P(x \mid B)}{P(y \mid B)}.\]
\end{ass}
In other words, IIA states that the composition of a choice set does not affect the relative attractiveness of items. 
A main question in this work will be, given a dataset $\mathcal{D}$ of choices from choice sets, 
can we determine whether $\mathcal{D}$ was generated by a model satisfying IIA or a model of a more general choice system? If not generated by a model satsify IIA, is it possible to define tractable model classes between the class of models satisfying IIA and the class of fully general choice systems? Our answer to this question is a formal truncation of a general choice system that we call the {\it context dependent random utility model} (CDM).

\subsection{Context-dependent Random Utility Models}
A trivial model of a general choice system is the {\it universal logit model} \cite{mcfadden1977application}, which simply parameterizes the choice system object, defining utilities $u(x \mid C), \forall x \in C$, for each $C \in \allsubsets$ that can vary arbitrarily for every item, for every set. The choice probabilities for a universal logit model are then:
\begin{align*}
    P(x \mid C) = \frac{\exp(u(x \mid C))}{\sum_{y \in C} \exp(u(y \mid C))}.
\end{align*}
The above model exhibits scale-invariance on each subset $C$, and thus we require that $\sum_{y \in C} u(y \mid C) = 0$, $\forall C \in \allsubsets$ for the purposes of identifiability. While relatively uninteresting as a model, the above formulation is the starting point for the following observation about choice systems, first documented by \citet{batsell1985new}.
\begin{lemma}
\label{lemma:bp}
The utilities in the universal logit model, $u(x \mid C)$, $\forall C \in \allsubsets, \forall x \in C$, can be uniquely mapped as
$$u(x \mid C) = \sum_{B \subseteq {C \setminus x}} v(x \mid B),$$ 
where $v(x \mid B)$ are values that satisfy the constraints $\sum_{x \notin B}  v(x \mid B) = 0$, $\forall B \subset \mathcal{X}$.
\end{lemma}
For greater clarity, we expand out the terms individually.
\begin{align*}
    u(x \mid C) = 
    &\underbrace{v(x)}_{\text{1st order}} + 
    \underbrace{\sum_{y \in C \setminus x} v(x \mid \{y\})}_{\text{2nd order}} +  \\
    & \underbrace{\sum_{\{y,z\} \subseteq C \setminus x} v(x \mid \{y,z\})}_{\text{3rd order}} +
     \ldots + \underbrace{v(x \mid C \setminus \{ x \} )}_{|C|\text{th order}},
\end{align*}
and expand out the first three sets of constraints:
\begin{align*}
\sum_{x \in \mathcal{X}} v(x) = 0,
\ \ 
\sum_{x \in \mathcal{X} \setminus y} v(x \mid \{y\}) = 0,
    \\
\sum_{x \in \mathcal{X} \setminus \{y,z\}} v(x \mid \{y,z\}) = 0 .
\hspace{1cm}
\end{align*}
We use $v(x) = v(x \mid \emptyset)$ for simplicity. 
This expansion reveals that arbitrary contextual utilities can be decomposed into intuitive contributions: the first order terms represent the item's intrinsic contribution to the utility, the second order terms represent the contextual contributions from every other item in the choice set, the third order terms the contributions from contextual pairs not modeled by contributions of the pairs constituent items, and so on. The expansion invites one to consider a hierarchy of choice model classes indexed by their order\footnote{This expansion differs from the one used by \citet{batsell1985new}, which expands the log probability ratios of items being chosen instead of the underlying contextual utilities.}: the $p$th order model refers to forcing all terms of order greater than $p$ to zero. Denote this class of choice system models by $\mathcal{M}_p$. Clearly, we have $\mathcal{M}_1 \subset \mathcal{M}_2 \subset \ldots \subset \mathcal{M}_{n-1}$, where $\mathcal{M}_{n-1}$ is the universal logit model. We next consider two excercises. First, we write out the $1$st order model explicitly:
\begin{align*}
    P(x \mid C) = \frac{\exp(v(x))}{\sum_{y \in C} \exp(v(y))}, \ \ \ \sum_{x \in \mathcal{X}} v(x) = 0.
\end{align*}
This model $\mathcal M_1$ is the multinomial logit model, the workhorse model of discrete choice \cite{luce1959individual,mcfadden1980econometric}. Moving to higher-order models, a counting exercise reveals that the number of free parameters in the $p$th order model is 
\begin{align*}
   \sum_{q=1}^p\binom{n}{q-1}(n-q),
\end{align*}
which simplifies to $n-1$ and $(n-2)2^{n-1}+1$ parameters for $\mathcal M_1$ (MNL) and $\mathcal M_{n-1}$ (universal logit), respectively. Clearly, the parameters grow polynomially in the number of items and exponentially in the order: there are $O(n)$ parameters for the 1st order model, $O(n^2)$ parameters for the 2nd, and so on.

The principled next step, then, is to consider the minimal model class that accounts for context effects, $\mathcal{M}_2$: 
\begin{align*}
    P(x \mid C) = \frac{\exp(v(x) + \sum_{z \in C \setminus x} v(x \mid \{z\}) )}{\sum_{y \in C} \exp(v(y) + \sum_{z \in C \setminus y} v(y \mid \{z\}) )},
    \\
    \text{s.t. } \sum_{x \in \mathcal{X}} v(x) = 0, \ \ \sum_{x \in \mathcal{X} \setminus y} v(x \mid \{y\}) = 0, \ \ \forall y \in \mathcal{X}.
\end{align*}
We remove most constraints (see Appendix \ref{section:constraints} for details) by introducing new parameters 
$u_{xz} = v(x \mid \{z\}) - v(z)$, $\forall x,z \in \mathcal X$, 
interpretable as the pairwise push and pull of $z$ on $x$'s utility. 
We may then rewrite the above, $\forall C \subseteq \mathcal X, \forall x \in C,$ as
\begin{align}
\label{eqn:cdm}
    P(x \mid C) = \frac{\exp(\sum_{z \in C \setminus x} u_{xz})}{\sum_{y \in C} \exp(\sum_{z \in C \setminus y} u_{yz})},
\end{align}
with one constraint, $\sum_{x \in \mathcal{X}} \sum_{y \in \mathcal{X} \setminus x} u_{xy} = 0$. We may then do away with this final constraint by noticing that $P(x \mid C)$ is invariant to a constant shift of $u_{xz}$, as in the case of both the MNL and universal logit model.
We refer to the model $\mathcal M_2$, as parameterized in equation~\eqref{eqn:cdm}, as the {\it context dependent random utility model (CDM)}, and note that it has $n(n-1)-1$ free parameters.

The CDM then corresponds to the following restriction on choice systems. 
\begin{ass}[Pairwise linear dependence of irrelevant alternatives]
\label{ass:pairiia}
A choice system on $\mathcal X$ satisfies pairwise linear dependence of irrelevant alternatives if, in the universal logit representation of Lemma~\ref{lemma:bp}, $v(x \mid B) = 0$ for all $B \subset \mathcal X$ for which $|B| \ge 2$.
\end{ass} 

This assumption can either be taken literally,
or can be justified as an approximation on the grounds of applications: in practice many problems are concerned with choices from relatively small sets, and the linear context effect assumption is then a decent approximation.

\subsection{Low-rank CDMs}
From equation~\eqref{eqn:cdm}, it is clear that the parameters of the CDM, $u_{xz}$, $\forall x,z \in \mathcal X$, have a matrix-like structure. Note that the parameters do not quite form a matrix, as the diagonal elements $u_{xx}$ are undefined and unused. But given this structure, it is natural ask if the pairwise contextual utilities can be modeled by a lower-dimensional parameterization. 

Formally, we define the low-rank CDM as a CDM where the pairwise contextual utilities jointly admit a low-rank factorization 
$u_{xz} =  c_z^T t_x, \forall x,y \in \mathcal X.$
We call $t_x, c_x \in \mathbb R^r$, the {\it target} and {\it context} vectors, respectively, for each item $x \in \mathcal X$. We can then write the choice probabilities of the low-rank CDM, for all  $C \subseteq \mathcal X$, for all $x \in C$, as:
\begin{align}
\label{eqn:lowrankcdm}
    P(x \mid C) = \frac{\exp((\sum_{z \in C \setminus x} c_z)^T t_x )}{\sum_{y \in C} \exp((\sum_{z \in C \setminus y} c_z)^T t_y )}.
\end{align}
The rank-$r$ CDM then has $2nr$ parameters and has at most $\min \lbrace (2n-r)r, n(n-1) - 1\rbrace$ degrees of freedom.

Our low-rank assumption is strongly related to standard additive utility models where one is given a low-dimensional featurization $x \in \mathbb R^r$ of each item and an individual's utility is $\beta^T x$. A difference here, other than the notion of contextual utility, is that we assume no featurization is available and that it must be learned. 

\section{Identifiability and Estimation of the CDM}

Consider a dataset $\mathcal{D}$ of choices with generic element $(x, C)$ that correspond to observing element $x$ being chosen from set $C$. Let $\uniqueD$ denote the collection of {\it unique} subsets of $\mathcal{X}$ represented in $\mathcal{D}$. If we assume that the data was generated by a CDM, it is important to understand conditions under which the parameters of that CDM are identifiable and conditions under which the expected error of a tractable estimation procedure converges to zero as the dataset gets large. In this section we furnish two sufficient conditions and one ``insufficient'' condition for identifiability. We then bound the expected squared $\ell_2$ error of the maximum likelihood estimate (MLE) of a full-rank CDM. Because the log-likelihood of the full-rank CDM is convex (by the convexity of log-sum-exp), we know we can efficiently find this MLE, and that this bound on the error of the full-rank model also bounds the error of any low-rank model.

We consider the dataset $\mathcal{D}$ as being generated in the following hierarchical manner:
\begin{enumerate}
\item A choice set $A$ is chosen at random from a distribution on the set of all subsets of $\mathcal X$.
\item The chooser chooses an item $x$ from the choice set $A$ according to a CDM with parameters $\theta \in \Theta$.
\end{enumerate}

We can parametrize the utility function by $\theta$ referring to it as $u_{\theta}(\cdot \mid \cdot).$ Given a $\mathcal{D}$ and guess $\theta$ we can write the probability of $(x, A)$ as 
$$
P_{\theta} (x \mid A) = \frac{\exp(u_{\theta} (x \mid A))}{\sum_{y \in A} \exp(u_{\theta} (y \mid A))}.
$$
This means we have a well defined likelihood function for the full dataset 
\begin{align}
\label{eqn:cdmlike}
\mathcal{L} (\mathcal{D} \mid \theta) = \prod_{(x , A) \in \mathcal{D}} P_{\theta} (x \mid A).
\end{align} 

For now we consider a full rank CDM where the parameter vector $\theta$ is the set of pairwise contextual utilities $u_{xz}$, $\forall x, z \in \mathcal X$. We will consider $u \in \mathbb R^d$ as the parameter vector, where for the full-rank CDM $d=n(n-1)-1$. Because $u$ can only be identified up to a scale, we consider possible CDMs with the constraint that $\sum_{xz} u_{xz} = 0$.

The likelihood \eqref{eqn:cdmlike} can be maximized using standard techniques. 
We will say that a dataset {\it identifies} a CDM if there are no two sets of parameters that have the same distribution $P(x \mid C), \forall C \in \allsubsets, \forall x \in C$.
We now give a sufficient (but not necessary) condition for identification.
\begin{thm}
\label{thm:two_sets_good}
A CDM is identifiable from a dataset $\mathcal{D}$ if $\uniqueD$ contains comparisons over all choice sets of two sizes $k, k'$, where at least one of $k,k'$ is not $2$ or $n$.
\end{thm}

The proof is given in Appendix~\ref{app:id}. In the multinomial logit model, the constraints of IIA allow us to identify all parameters given just the probability distribution $P(\cdot \mid \mathcal{X})$, but in the less constrained CDM more information is needed. 
For a simple demonstration of the theorem, consider a choice system on $\mathcal{X} = \lbrace a, b, c \rbrace$ where
\begin{align*}
P(a \mid \mathcal{X}) = 0.8, \ \ 
P(b \mid \mathcal{X}) = 0.1, \ \ 
P(c \mid \mathcal{X}) = 0.1. 
\end{align*}
Here, if we assume IIA we can infer any pairwise choice probability simply by taking the appropriate ratio. However, if we do not assume IIA and only assume Assumption~\ref{ass:pairiia} (equivalent to assuming the choice system is a CDM), \textit{any} set of pairwise probabilities is consistent with what we've observed. Thus the CDM is not identified if we only receive data about choices from $\{a,b,c\}$. Moreover, because CDM can fit any pairwise probability while retaining the above choice probabilities over $\{a,b,c\}$, the CDM clearly violates IIA. We elaborate further in Appendix \ref{section:examples}, showing specific CDM parameters handling violations of IIA, and further discuss the various context effects (e.g. the compromise effect) the CDM can discover and accommodate.

In Appendix~\ref{app:id} we show that the identifiability of the full-rank CDM for a given dataset $\mathcal{D}$ is equivalent to testing the rank of an integer design matrix $G(\mathcal{D})$ constructed from the dataset (Theorem~\ref{thm:rank_test}). This characterization of the identifiability of the full-rank CDM also gives a sufficient condition for the identifiability of low-rank CDMs. The proof of this theorem can be easily expanded to demonstrate an advantage of the CDM instead of a general choice system: for a general choice system to be identified $\uniqueD$ would need to include in its support \textit{every} choice set.

In addition to the above sufficient conditions for identifiability, we also have the following result about an important ``insufficient'' condition.

\begin{thm}\label{thm:single_set_bad}
No rank $r$ CDM, $1 \le r \le n$, is identifiable from a dataset $\mathcal D$ if $\uniqueD$ contains only choices from sets of a single size.
\end{thm}

The proof is given in Appendix~\ref{app:id}. 
Requiring comparisons over two different choice set sizes is not unique to the CDM; recent results \cite{chierichetti2018learning} demonstrate that even a uniform mixture of two multinomial logit models, a special case of the mixed logit that violates IIA, requires comparisons over two different choice set sizes.

As a result of this theorem, for choice data collected from sets of a fixed size, the parameters of a CDM model that has been fit to data can not be interpreted without some amount of explicit or implicit regularization. This non-identifiability also applies to all blade-chest models \cite{chen2016modeling}, which (as alluded to in Section~\ref{sec:unify}) are CDM models restricted to pairwise choices. We further explore regularization and identifiability in Appendix \ref{section:regl}.

\subsection{Uniform Performance Guarantees}

The likelihood function is log-concave and can thus be solved to arbitrary error through standard convex optimization procedures (avoiding shift invariance with the constraint $\sum_{xz} u_{xz} = 0$). We now show that maximum likelihood estimation efficiently recovers the true CDM parameters under mild regularity conditions.

\begin{thm}\label{thm:CDM_bound}
Let $u^\star$ denote the true CDM model from which data is drawn. Let $\hat{u}_{\text{MLE}}$ denote the maximum likelihood solution. Assume $\uniqueD$ identifies the CDM. For any $u^\star \in \mathcal{U_B} = \lbrace u \in \mathbb{R}^d : \left\lVert u\right\rVert_\infty \leq B, \textbf{1}^Tu = 0 \rbrace$, and expectation taken over the dataset $\mathcal{D}$ generated by the CDM model,
$$
\mathbb{E}\big[\left\lVert 
\hat{u}_{\text{MLE}}(\mathcal D) - u^\star \right\rVert_2^2\big] 
\leq 
c_{B,k_{\text{max}}}\frac{d-1}{m},
$$
where $k_{\text{max}}$ refers to the maximum choice set size in the dataset, and $c_{B,k_{\text{max}}}$ is a constant that depends on $B$, $k_{\text{max}}$ and the spectrum of the design matrix $G(\mathcal{D})$.
\end{thm}

The proof is given in Appendix~\ref{app:convergence}, where we also 
state the exact relationship of $c_{B,k_{\text{max}}}$
to the max norm radius $B$, the maximum choice set size $k_{\text{max}}$ and design matrix $G(\mathcal{D})$. 
Both the identifiability condition and maximum norm bound are essential to the statement, as the right hand side diverges when the former is violated, and diverges as $B \rightarrow \infty$.

Theorem~\ref{thm:CDM_bound} is a generalization of a similar convergence result previously shown for the multinomial logit case \cite{shah2016estimation} (the multinomial logit model class is a subset of the CDM model class). The proof follows the same steps, showing first that the objective satisfies a notion of strong convexity, and using that fact to bound the distance between the estimate and the true value. Our contributions augment the notation of \citet{shah2016estimation} to support multiple set sizes and the more complex structure of the CDM, and carefully bound the role of these deviations in the steps leading to the result. 



To our knowledge, this convergence bound furnishes the first tractable sample complexity result for a model that can accommodate deviations from a random utility model (RUM). A comparable lower bound, which we do not furnish in this work, would make clear whether the maximum likelihood procedure is inferentially optimal or not. 
And while stated for the full-rank model, our convergence bound holds for CDMs of any rank. It is possible that low-rank CDMs admit an improved rank-dependent convergence rate.

\subsection{Testing}
We can use the CDM to construct a statistical test of whether our data is indeed consistent with the MNL/Luce model, and thus IIA, across the choice sets we observe. Recall that the class of Luce models is nested within the CDM, which is in turn nested within the universal logit, as discussed in Section~\ref{sec:choicesystem}. We can consider the following likelihood ratio statistic,
\[\Lambda(\mathcal{D}) = \frac{\sup_{\theta \in \Theta_{\text{Luce}} \subset \Theta_{\text{CDM}}} \mathcal{L}(\mathcal{D} \mid \theta)}{\sup_{\theta \in \Theta_{\text{CDM}}} \mathcal{L}(\mathcal{D} \mid \theta)},\]
where $\Theta_{\text{Luce}}$ and $\Theta_{\text{CDM}}$ respectively refer to the parameter classes of Luce and CDM Models. We then appeal to a classical result from asymptotic statistics \cite{wilks1938large} that as the sample size $m \rightarrow \infty$, $D = -2 \log(\Lambda(\mathcal{D}))$ converges to the $\chi^2$ distribution with degrees of freedom $\Delta$ equal to the difference between the number of parameters between the two model classes. For CDM and Luce, $\Delta = n(n-2)$. For a universal logit and Luce, $\Delta = (\sum_{C \in \uniqueD} (|C| - 1))  - (n - 1)$, where $\uniqueD$ are the unique subsets in the dataset that the test can reasonably evaluate. Our test then compares the statistic to the value of the $\chi^2_{\Delta}$ distribution corresponding to a desired level of statistical significance.

We are keen to note that the CDM test likely enjoys finite sample guarantees when the true distribution is sampled from a CDM, 
owing to the vanishing risk of the MLE shown in Theorem~\ref{thm:CDM_bound}. In experiments that follow, we look at the finite sample performance of this likelihood ratio test, evaluating this claim empirically and comparing the performance of our test to the universal logit test.

\subsection{Unifying Existing Choice Models}
\label{sec:unify}

The CDM and low-rank CDM generalize a number of prominent choice models in a unified framework. In this section we present connections to the work of \citet{tversky1993context}, \citet{batsell1985new}, and \citet{chen2016modeling,chen2016predicting}. This means that our convergence and identifiability results carry over to these other models, which all previously lacked such results.
 
{\bf The Tversky-Simonson model.}
The {\it additive separable utility model (ASM)} is the cornerstone of random utility modeling in many applications. In the ASM the utility of item $x$ can be written as an inner product $ u(x) = w^T t_x$, where $t_x$ is a feature vector of item $x$ (typically known to the analyst, but sometimes latent) and the vector $w$ contains the parameters of the linear model (estimated from data). The parameters $w$ have a real world interpretation: they are the weights that an individual places on each attribute. These can be used to estimate counterfactuals: for example, how much would an individual rank a new item $y$ that we have not seen before?

A seminal experiment by Tversky and Kahneman asks individuals to consider a situation where they are purchasing an object and they learn that the same object is available across town (a 20 minute drive away) for $\$ 5$ cheaper \cite{tversky1985framing}. They then ask whether the individuals would drive across town to take advantage of this lower price, essentially a question about their value of time. Individuals are more likely to drive across town when they are considering purchasing a $\$ 10$ object compared to when they are purchasing a $\$ 120$ object, even though the time/money tradeoff is identical.

The ASM assumes that the weights are constant across contexts, making the choices in the story above impossible if the ASM is indeed the true model. \citet{tversky1993context} expand the ASM to allow context to adjust the weights that individuals place on attributes while keeping the attributes fixed. This approach has a particular psychological interpretation: the presence of certain items makes some dimensions of a choice more salient than others, an effect that appears across a variety of decision situations.

This is formalized by setting utility of $x$ in context $C$ to be $u^{TS} (x \mid C) = w(C)^T t_x.$ Tversky and Simonson discuss several ways in which some experimental results can be modeled using various forms of weights $w(C)$, though their approach requires both features and context-dependent weight functions to be hand-engineered. They do not formalize any procedure for how one can learn such a model from choice data directly. Thus our CDM can be seen as a method for learning the parameters of a Tversky-Simonson model directly from data in an efficient manner. 

{\bf The Batsell-Polking model.}
\citet{batsell1985new} introduces a model of competing product market shares that can also be written as a truncated expansion of the log ratio of choice probabilities. The CDM can be viewed as an alternative parameterization of a {\it third-order Batsell-Polking model}.
There are several significant differences between the way Batsell and Polking viewed their third-order model and how we view the CDM. First, Batsell and Polking advocated for fitting their models to data using a hand-tuned least squares procedure whereas we use more general maximum likelihood techniques. Second, our identifiability and convergence results are entirely new. Their least-squares procedure understandably has no analogous guarantees. Lastly, our restriction to low-rank parameterizations is squarely new and can greatly reduce the model complexity. 

{\bf The Blade-Chest model.} 
Standard models for competition build on the Elo rating system for chess \cite{elo1978rating} and the TrueSkill rating system for online gaming \cite{herbrich2006trueskill}. Both of these models assume that individuals have a one-dimensional latent ``skill'' parameter that can be discovered from matchup data between competitors.

The Blade-Chest model \cite{chen2016modeling,chen2016predicting} tries to model rock-paper-scissors-type intransitivies in pairwise matchups through a multidimensional latent embedding of skill. 
In the language of our CDM, the blade-chest ``inner product'' model (the authors also consider a ``distance'' model) defines the probability that $x$ beats $y$ as: 
\begin{equation*}
\Pr(x \mid \{x,y\}) = \dfrac{\exp( t_x^T c_y )
}{
\exp( t_x^T c_y ) +
\exp( t_y^T c_x )
},
\end{equation*}
which is precisely a CDM restricted to pairs.
We can view the CDM as a natural extension of the Blade-Chest model from pairs to larger sets. Considering our negative identifiability result for choice data consisting of only a single set size (Theorem~\ref{thm:single_set_bad}), we conclude that the Blade-Chest model is not identifiable and requires either explicit or implicit regularization in order to make the parameters interpretable. 

\section{Experiments}

\begin{figure*}[!ht]
	\centering
	\includegraphics[width=.24\textwidth]{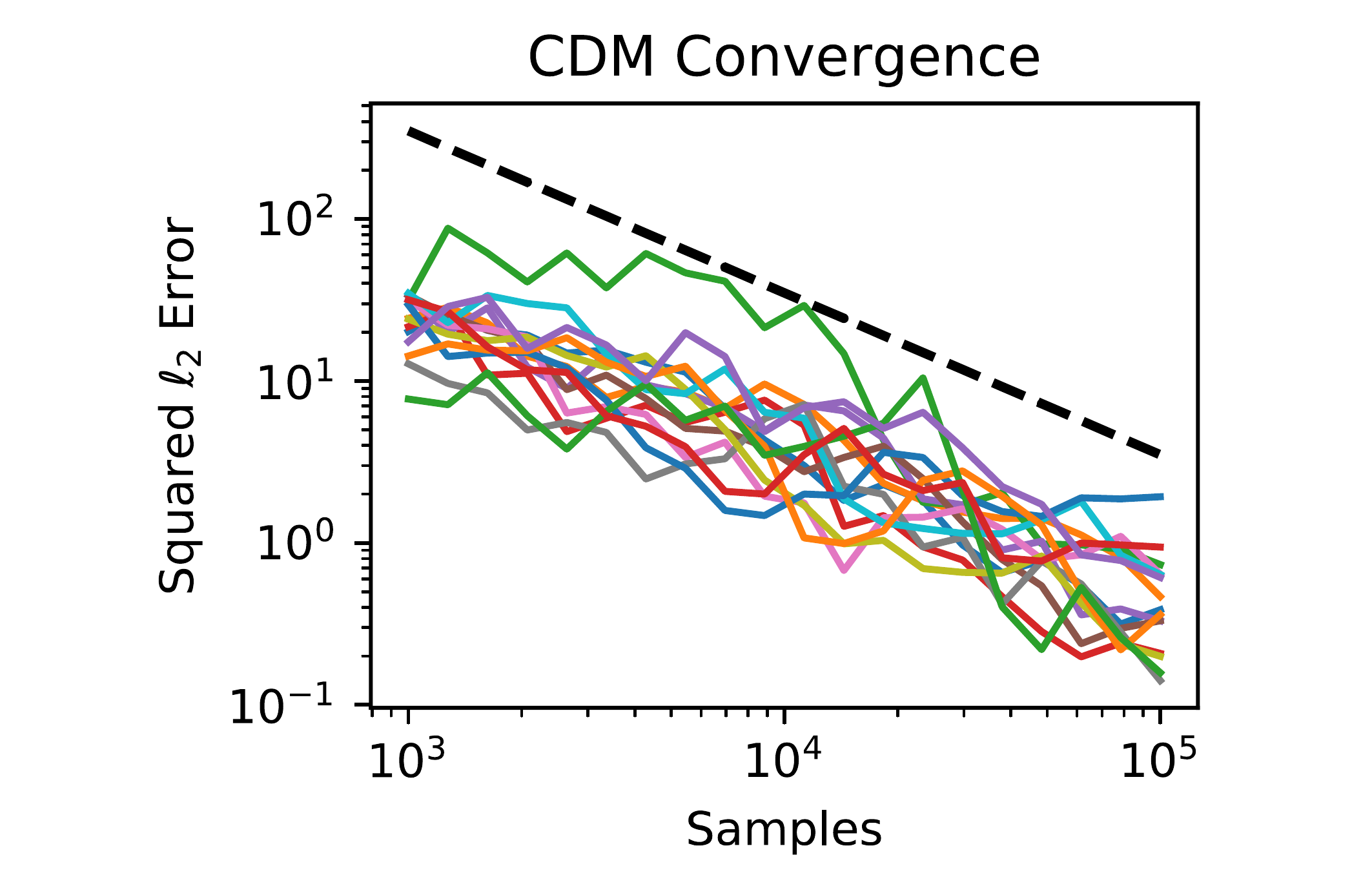}
	\includegraphics[width=.24\textwidth]{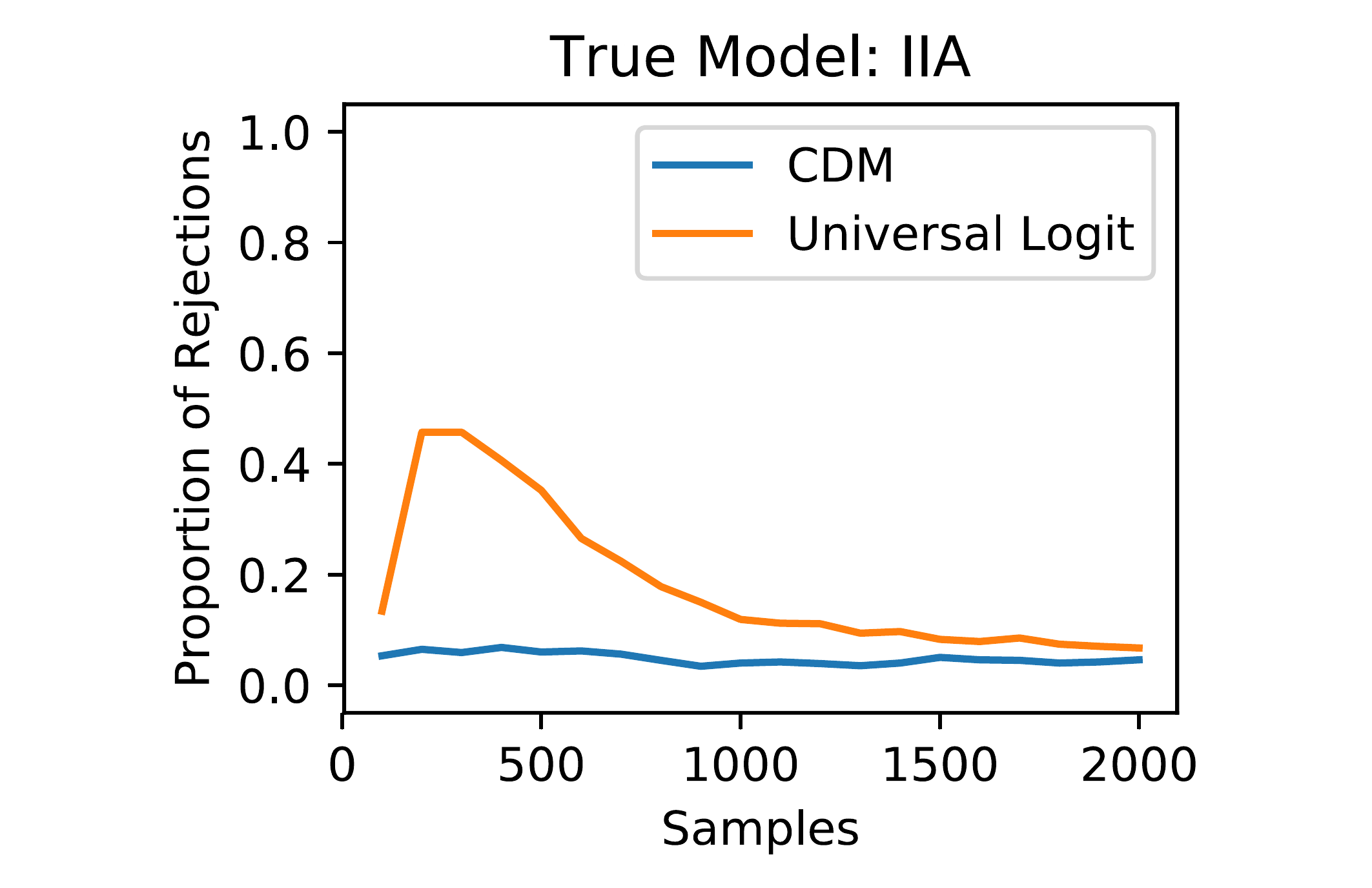}
	\includegraphics[width=.24\textwidth]{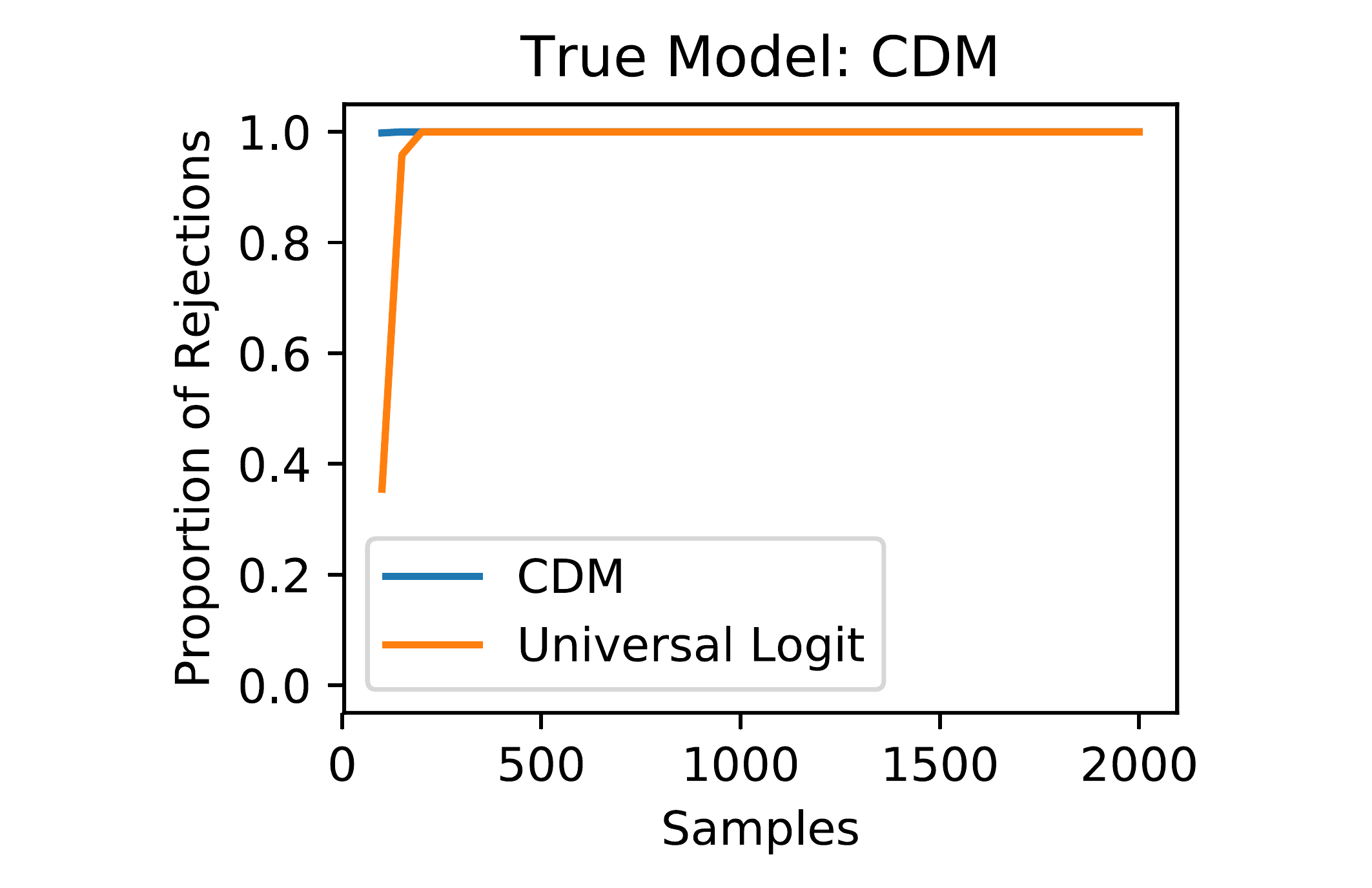}
	\includegraphics[width=.24\textwidth]{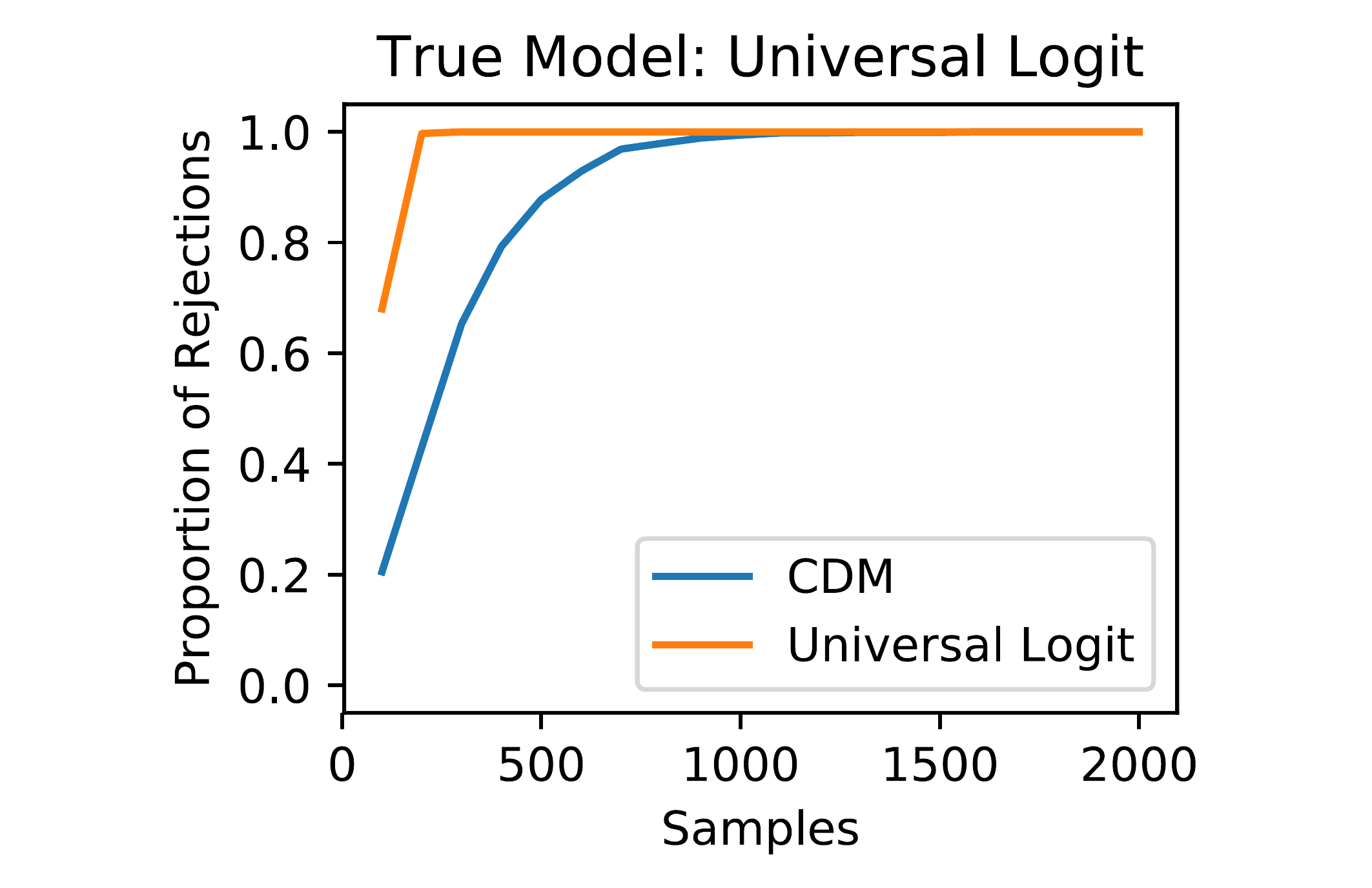}
	\vspace{-4mm}
	\caption{{\it (a)} Approximation error of an estimated CDM in $10$ growing datasets validates our convergence theorem. The dashed black line is a visual guide of the slope $1/m$. {\it (b,c,d)} The proportion of rejections for a CDM-based hypothesis test  of IIA (at a threshold of $p < .05$) when the data is generated by a MNL, CDM, and general choice system model as a function of the number of samples. When IIA is true, the CDM-based test has a 5\% of false rejection rate while the test based on the general choice system is highly anti-conservative. When IIA is false, both tests quickly and correctly reject. All model parameters are described in the main text.} 
	\label{fig:cdmerror}
\end{figure*}

We now evaluate the CDM and low-rank CDM on data. Our evaluation includes comparisons with MNL/Luce models and mixed MNL models \cite{mcfadden2000mixed}. MNL and CDM model likelihoods are optimized using Adam~\cite{kingma2014adam}, a stochastic gradient descent algorithm with adaptive moment estimation. Mixed MNL likelihoods are optimized using open source code from \cite{ragain2016pairwise}. The CDM parameter optimization is initialized with values corresponding to a Luce MLE for that dataset. All datasets are pre-existing and public; replication code for all figures will be released at publication time.

{\bf Simulated Data.} 
We begin with simulated data, which allows us to validate our theoretical results regarding the convergence of the MLE in a setting where the underlying data-generating process is known. Since we know whether IIA holds in the simulated data, simulated data is also useful for examining two aspects of the CDM-based hypothesis test. First, we ask about the power of the test, in other words, does the test reject IIA when it is not true? Second, we ask about the conservatism of the test. The nested model likelihood ratio tests are only valid asymptotically; in our simulated data we can check whether the CDM over or under-rejects the null hypothesis of IIA in finite samples.

We consider three data-generating processes: one where the data is generated from a MNL model (where IIA holds), one where the data is generated from a CDM, and one where the data is generated from a general choice system. The universe has $n=6$ items. In the IIA dataset the underlying MNL model has parameters $[1,2,3,4,5,6]/21$. In the CDM dataset the parameters $U = T^T C$ are generated by sampling elements of both 6x6 matrices $T$ and $C$ i.i.d.\ from $N(0,1)$. The probabilities of the general choice system are sampled $U[0,1]$ and renormalized.
We sample choice sets uniformly at random (thus our identification conditions are quickly met) and then a choice according to the underlying model. We fit a Luce, CDM, and universal logit model to the data and look at both the error of the CDM MLE (to evaluate convergence) and the $p$-value from the nested model likelihood ratio tests. When the $p$-value falls below $.05$ we say that the hypothesis of IIA is rejected. 

In addition, we compare the CDM-based nested test to another nested model test where we use a general choice system as the alternative model. Recall that the general choice system also nests MNL. However, the general choice system has combinatorially more parameters.

Figure~\ref{fig:cdmerror} shows our results. The left panel validates the $O(\frac{1}{m})$ convergence result in Theorem~\ref{thm:CDM_bound}. The right three panels look at how often the hypothesis of IIA is rejected, out of $1000$ independent growing datasets, when the underlying data comes from the three different data generating processes. We see that the CDM rejects the null less than $5 \%$ of the time when the data generating process indeed satisfies IIA and rejects IIA when it is not true almost all the time, even with relatively small amounts of data. 

By contrast we see that the universal logit requires quite a lot of data to reach the asymptotically valid coverage, even with a universe of only $6$ items. For finite samples it is highly anti-conservative, over-rejecting when IIA is true for small and medium amounts of data.

{\bf SFwork/SFshop.}
We now turn to two real-world datasets: SFwork and SFshop. These data are collected from a survey of transportation preferences around the San Francisco Bay Area \cite{koppelman2006self}. SFshop consists of 3,157 observations of a choice set of transportation alternatives available to individuals traveling to and from a shopping center, as well as what transportation that individual actually chose. SFwork is similar, containing 5,029 observations consisting of commuting options and the choice made.

\begin{figure}
\centering
\includegraphics[height=26mm]{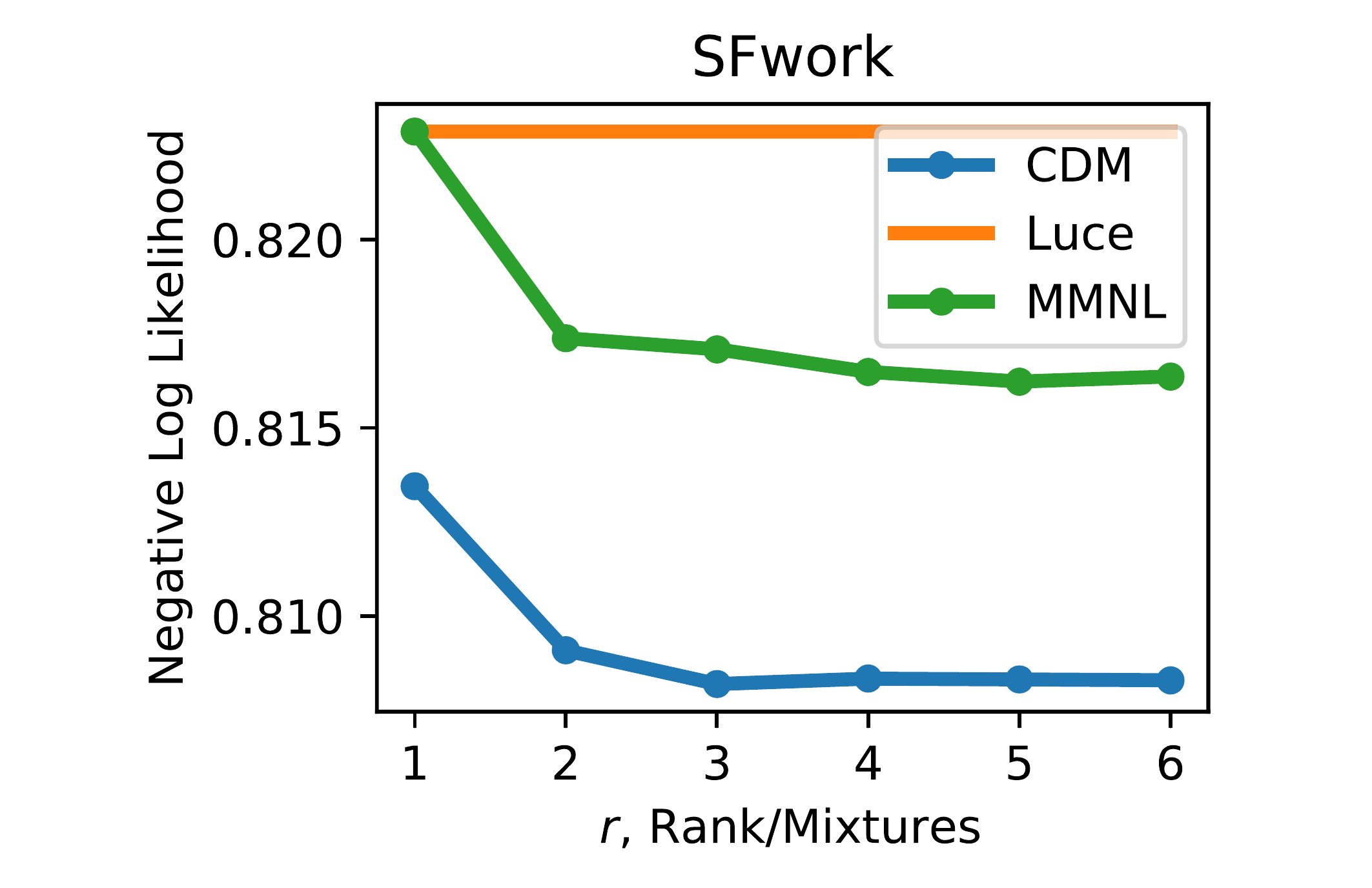}
\includegraphics[height=26mm]{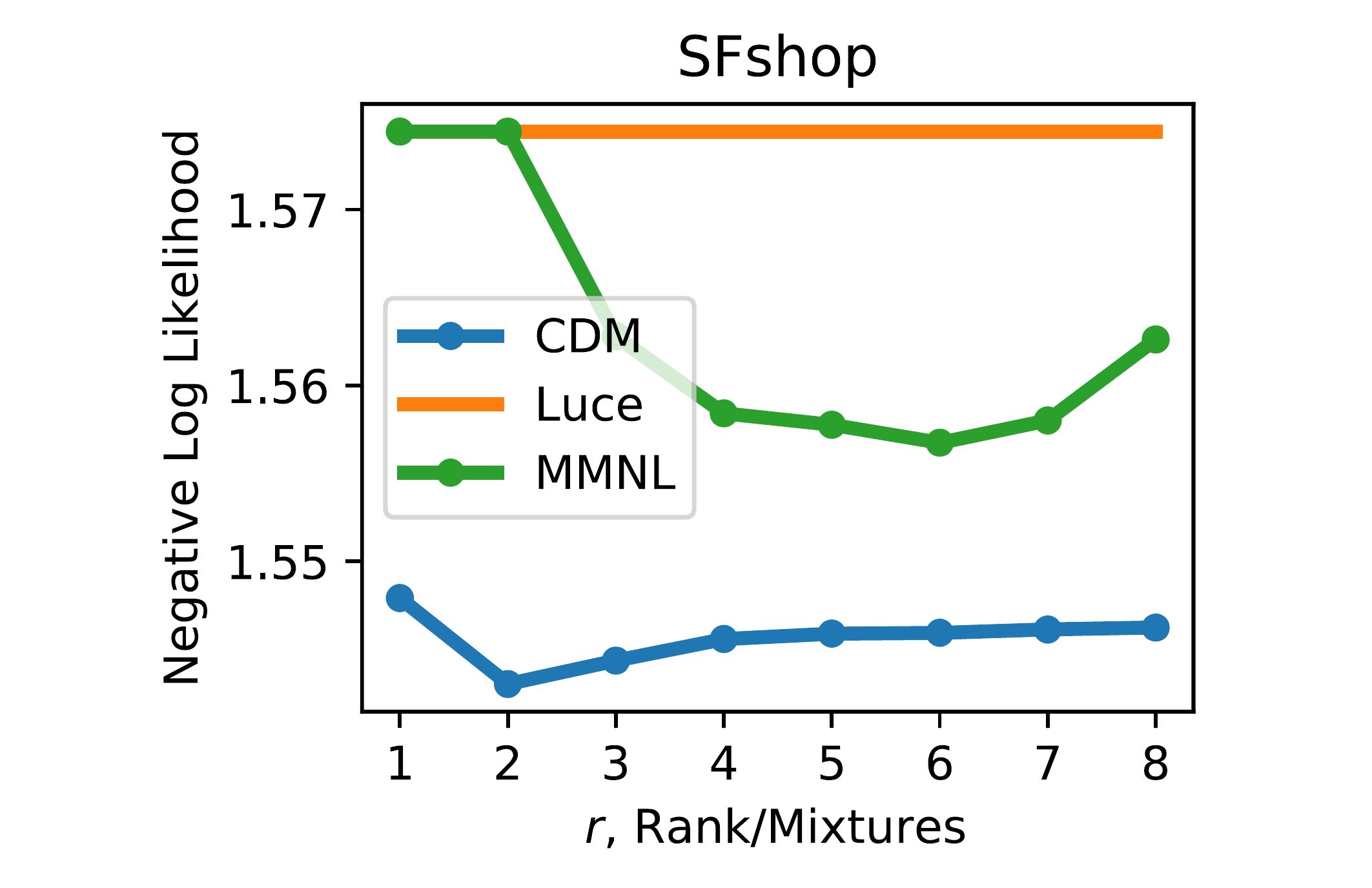}
\caption{
The out of sample negative log-likelihood of the MLE for the SFWork and SFshop datasets under an MNL/Luce model, mixed MNL model with varying number of mixture components, and CDMs of varying rank. The CDM outperforms the other models at all ranks.}
\label{fig:sfworkshop}
\vspace{-4mm}
\end{figure}

\begin{figure*}[!t]
\centering
\includegraphics[width=0.27\textwidth]{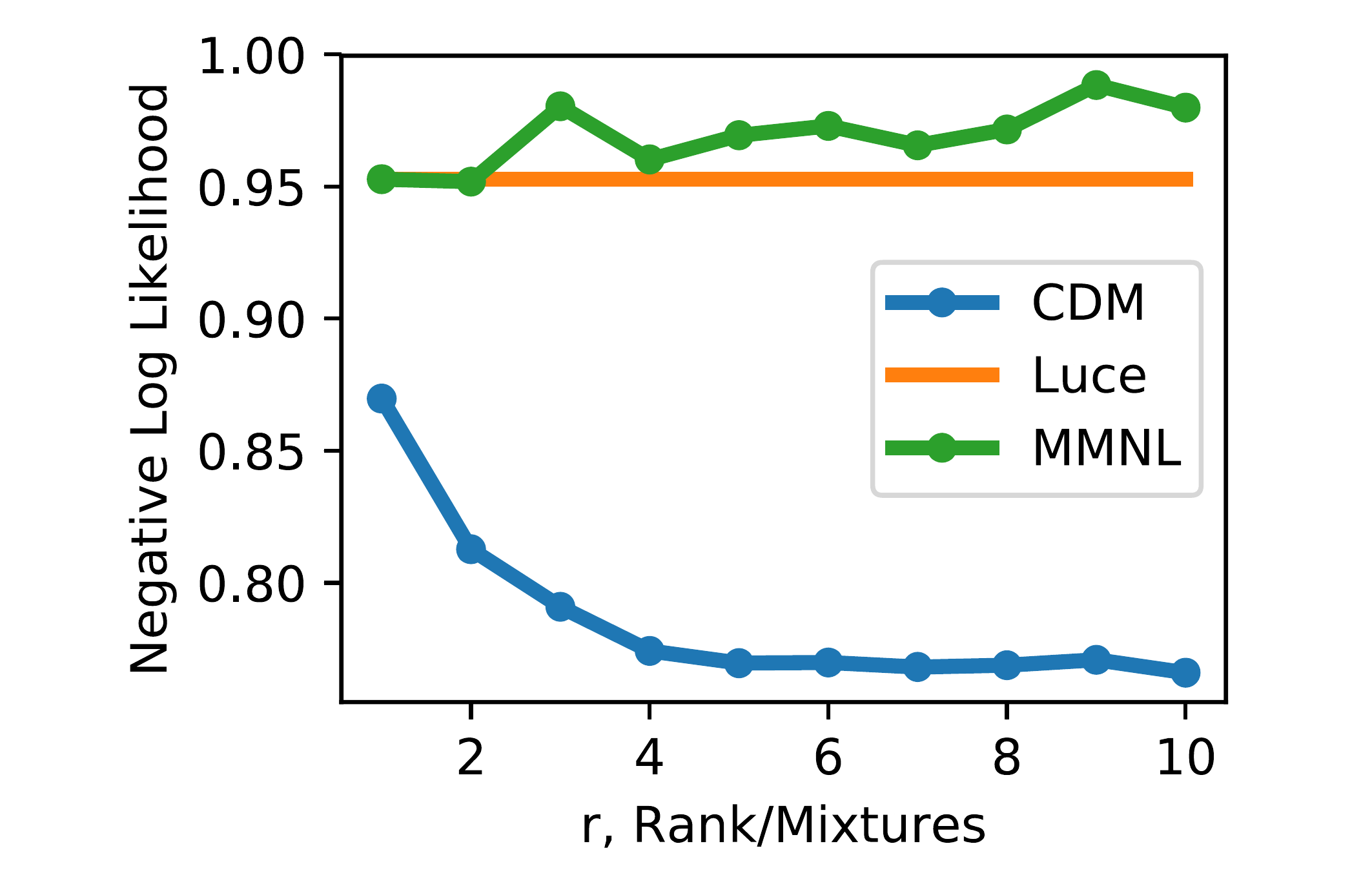}
\includegraphics[width=0.27\textwidth]{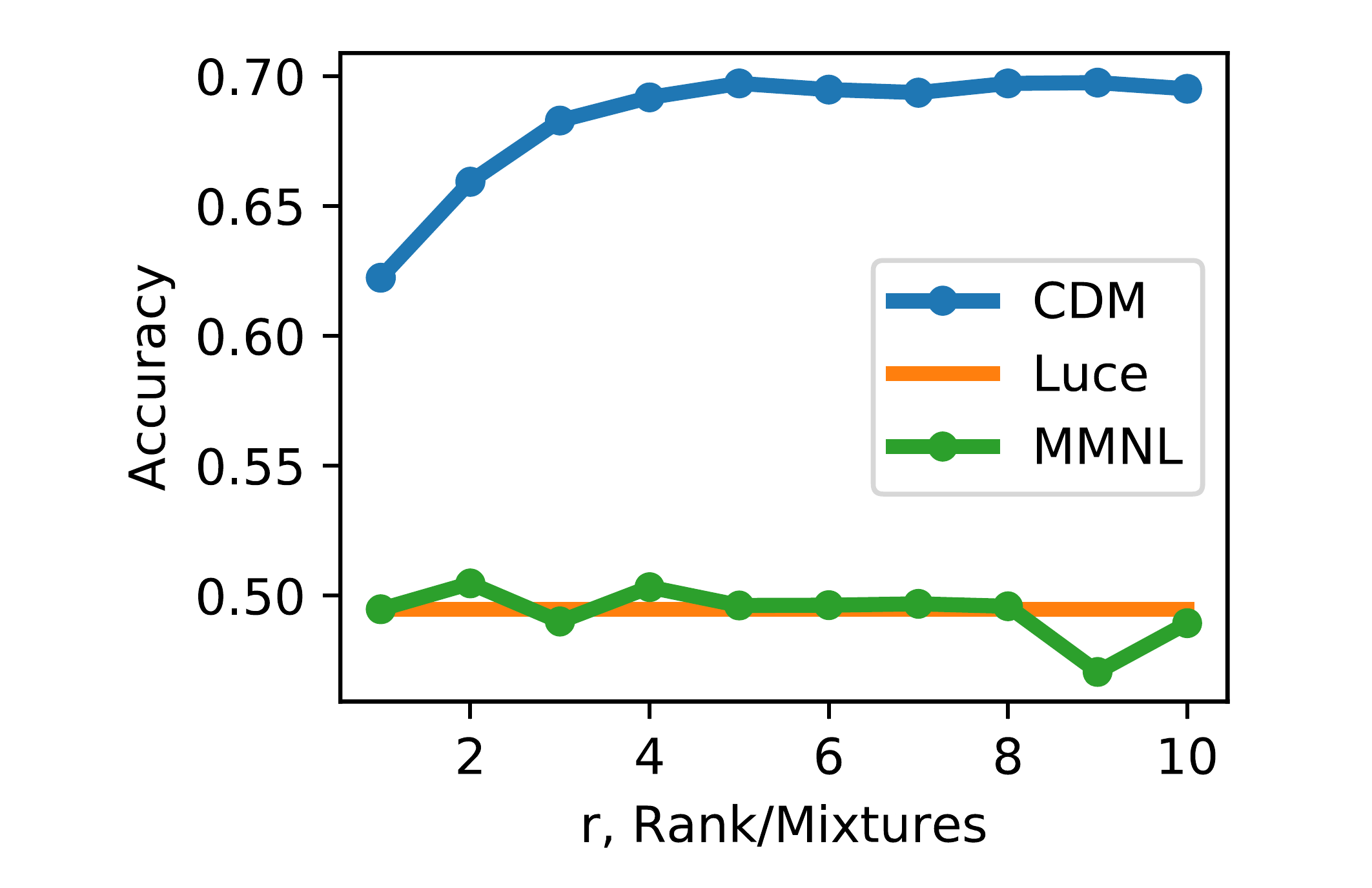}
\includegraphics[width=0.45\textwidth]{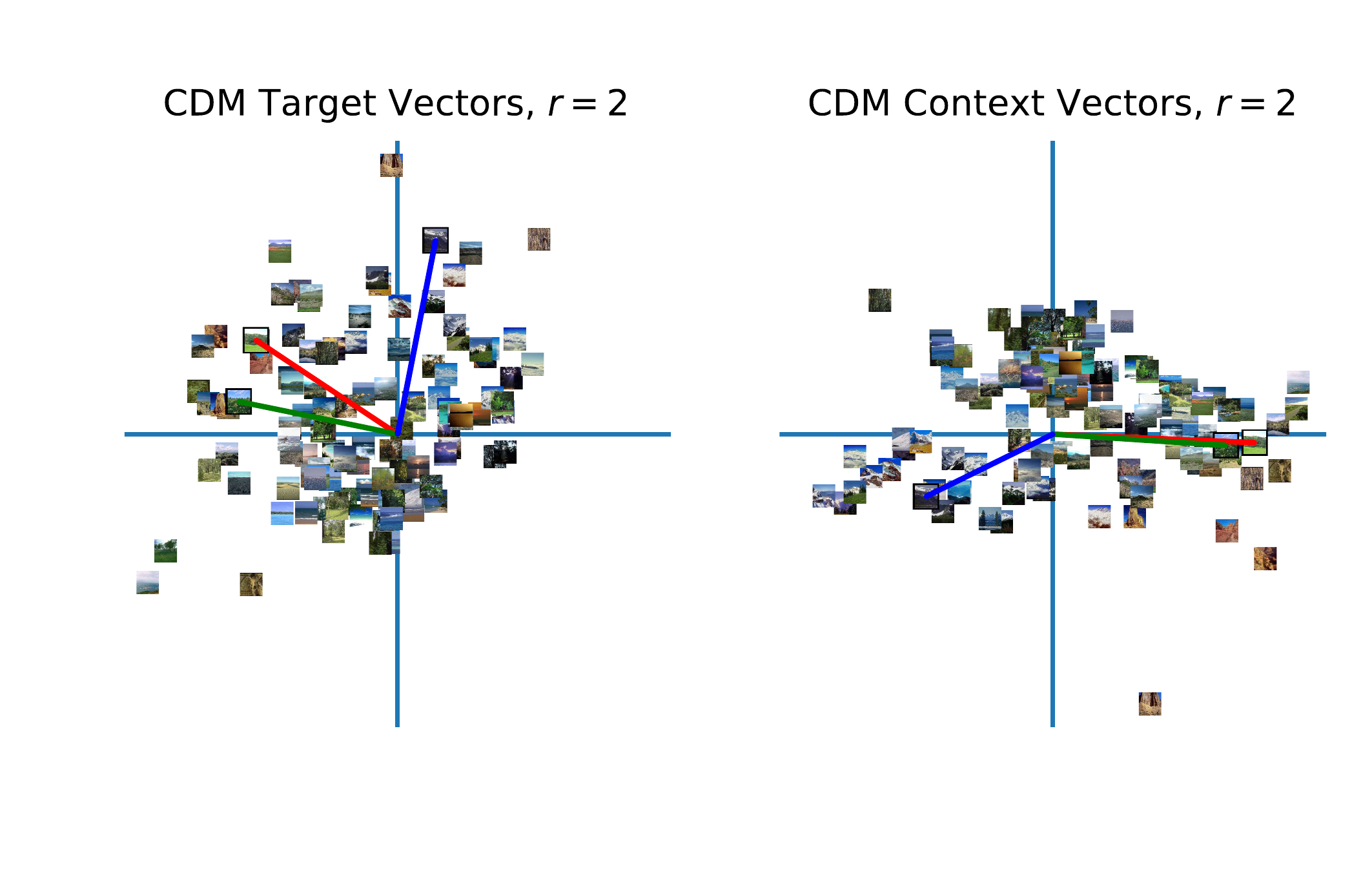}
\vspace{-7mm}
\caption{
(Left) The out of sample negative log-likelihood and accuracy of the MLE for the nature photo dataset under an MNL/Luce model, mixed MNL model with varying number of mixture components, and CDMs of varying rank. The CDM outperforms the other models at all ranks in terms of both likelihood and accuracy.
The target (center) and context (right) vector embeddings for the nature photo dataset with a rank-$2$ CDM, with three sample vectors highlighted. An item's context and target vectors have, on average, a negative dot product, showing that the addition of an item makes similar items already in the choice set less likely to be chosen.
}
\label{fig:sesame}
\end{figure*}

These datasets are similar to those employed in many demand estimation applications. For example, \citet{berry1995automobile} fit a MNL model to aggregate data in order to estimate the utility function of an average consumer for automobiles as well as how individuals (on average) trade off various qualities of a car (e.g., gas mileage vs.\ price). With access to underlying parameters, the analyst can then make counterfactual estimates such as, for example, what would be the sales of a hypothetical cheaper and higher gas mileage car? With the SFWork/SFShop data, we can ask questions like: what would happen if we made certain types of transit more or less available? Of course, if the underlying assumption (IIA) of the MNL model is wrong, then we may expect our counterfactual answers to also be wrong.

We run both hypothesis test for IIA (asking: ``is there IIA in the data?'') and examine the out-of-sample performance of the low-rank CDM (asking: ``does the violation of IIA have meaningful consequence for prediction?''). From the hypothesis test we obtain a $p$-value of $10^{-7}$ and can strongly reject IIA. Figure~\ref{fig:sfworkshop} shows the out of sample fit on a held out $20 \%$ of the data for low-rank CDMs, mixed MNLs, and an MNL model, again showing that IIA is not satisfied in this data. The full-rank unfactorized CDM is omitted from the figure for improved visibility, but attains an out of sample log-likelihood of $0.808$ and $1.540$, respectively for SFwork and SFshop. Though the unfactorized CDM outperforms MNL and mixed MNLs, it fails to outperform low-rank CDMs. 

{\bf Not Like The Other.}
We turn to a slightly different dataset to demonstrate another way the CDM can be used. We consider the task introduced by \citet{heikinheimo2013crowd} where individuals are shown triplets of nature photographs and asked to choose the one photo that is most unlike the other two. This task involves comparison-based choices \cite{kleinberg2017comparison} where there are no ``irrelevant alternatives'' and IIA is clearly violated: consider two example task where the choice set is two mountains and a beach vs.~two beaches and a mountain.

The dataset is comprised of $m=3355$ triplets spanning $n=120$ photos. 
Because the dataset only has choice sets of a fixed size, the CDM is not directly identifiable (Theorem~\ref{thm:single_set_bad}). We resolve this issue by adding an $\ell_2$ regularization term to the log-likelihood. 
For a small positive regularization penalty, the optimizer then selects the least norm solution within the null space. We choose a non-negligible penalty, chosen through cross-validation, to serve the additional purpose of improving model generalization.

We fit low-rank CDMs and see that they handily outperforms a MNL model (i.e.~just item-level utilities) and mixed MNL models on a $20 \%$ held-out test set (Figure~\ref{fig:sesame}, left). Though mixed MNL is often a competitive baseline, it is still a RUM, and cannot model the inherent asymmetric dominance of this task. The full-rank, unfactorized CDM is again omitted, but attains an out of sample log-likelihood of $0.843$, yet again outperforming MNL and mixed MNL but falling significantly short of the low rank models. We plot the vectors learned in the low-rank CDM (Figure~\ref{fig:sesame}, right). We see that similar images are grouped together both as targets and as contexts. We also see an intuitive property of the dataset: for most items $x$, $t_x$ and $c_x$ have a negative inner product. Essentially, having two copies of the same item in a choice set makes each copy less likely to be chosen. 

\section{Conclusion}

Existing work has argued that context dependence, and in particular choice-set dependence, is an important part of human decision-making \cite{ariely2003coherent,slovic1995construction,tversky1993context}. Tractable tools like the CDM are therefore crucial to further understanding decision-making, providing both good empirical performance and optimistic worst case guarantees. It should also be noted that IIA violations are often seen in intertemporal choice, choice under uncertainty, and choices about cooperation. Applying a CDM to these domains is an important area of future work.

There is separate experimental evidence that human choices are intransitive in some settings, where people may prefer $A$ to $B$ and $B$ to $C$ but then $C$ to $A$. This evidence has given rise to a theoretical literature on relaxing the transitivity axiom of rational choice or the regularity axiom of random utility \cite{tversky196919intransitivity, ragain2016pairwise, benson2016relevance}. To that end, an axiomatic characterization of the CDM and the kinds of violations of rational choice that the model can or cannot represent would be worthy of further study.


Understanding human decision-making is an important endeavor for both basic and applied science and is becoming increasingly important in human-centered machine learning and artificial intelligence. 
We view the introduction of techniques from machine learning and AI into behavioral science and the flow of realistic models of human behavior in the other direction as crucial and beneficial for both fields \cite{wager2015estimation,fudenberg2014recency,naecker2015lives,epstein2016good,peysakhovich2017group}. We hope that our work contributes to this important conversation.


\section*{Acknowledgments}

We thank Fred Feinberg, Stephen Ragain, and Steve Yadlowsky for their helpful comments and feedback. AS was supported in part by an NSF Graduate Research Fellowship. JU was supported in part by an ARO Young Investigator Award.

\bibliographystyle{agsm}
\bibliography{cdm}

\onecolumn
\appendix

\section{Proofs of Identifiability}
\label{app:id}

There are three main theorems proven in this section of the appendix. The first two are given in the main text.

{\bf Theorem~\ref{thm:two_sets_good}.} \ \emph{
A CDM is identifiable from a dataset $\mathcal{D}$ if $\uniqueD$ contains comparisons over all choice sets of two sizes $k, k'$, where at least one of $k,k'$ is not $2$ or $n$.
}

{\bf Theorem~\ref{thm:single_set_bad}.} \ \emph{
No rank $r$ CDM, $1 \le r \le n$, is identifiable from a dataset $\mathcal D$ if $\uniqueD$ contains only choices from sets of a single size.
}
\begin{thm}
\label{thm:rank_test}
A full rank CDM is identifiable from a dataset $\mathcal D$ if and only if the rank of an integer design matrix $G(\mathcal{D})$, properly constructed, is $n(n-1)-1$.  
\end{thm}

We begin with a few definitions and simple facts, providing proofs for clarity. Given these facts, main workhorse for proving our identifiability theorems is Lemma~\ref{lemma:rank_condition}. 

Since the CDM parameters are invariant to constant offsets, we choose (for the full rank case) an offset such that 
\begin{align}\label{offset_eq}
    \sum_{x \in \exx} \exp\Big(\sum_{z \in \exx \setminus x} u_{xz}\Big) = 1.
\end{align} 
Note that this implies $P_{x,\exx} = \exp(\sum_{z \in \exx \setminus x} u_{xz})$.\\

Because the CDM is a logit-based model, it will be much easier to work with log probability ratios. To that end, we define, for a choice set $C \ni x$, \begin{align} \label{eqn:betaxC}
\beta_{x, C} = \log(P_{x,C} / \bar{P_{C}}),
\end{align}
where $\bar{P_{C}} = (\prod_{y \in C}P_{y,C})^{\frac{1}{|C|}}$, the geometric average of the probabilities.
\begin{fact}\label{unique_beta_map_prob}
	Given a choice set $C$ of size $s$, there is a 1-to-1 mapping between the set of log probability ratios $\{\beta_{x,C}: x \in C \}$ and the set of probabilities $\{ P_{x,C}: x \in C\}$. 
\end{fact}
\begin{proof} Uniquely find $\beta_{x,C} \ \forall x \in C$ using the mapping in equation~\eqref{eqn:betaxC}. Now, for the other direction, observe that $\frac{\exp\beta_{x,C}}{\sum_{y \in C}\exp\beta_{y,C}} = \frac{P_{x,C}/\bar{P_{C}}}{\sum_{y \in C}P_{y,C}/\bar{P_{C}}} = P_{x, C} \ \forall x \in C$.
\end{proof}
Hence, statements regarding identifiability between CDM parameters and the $\beta$'s can be mapped to statements about identifiability between CDM parameters and probabilities. It will also be much easier to relate differences in CDM parameters of the following pattern, $u_{xy} - u_{yx}$ and $u_{xz} - u_{yz}$ $\forall x \neq y \neq z$, to the $\beta$'s. Because CDM is shift invariant, these differences between parameters uniquely identify the parameters when the offset constraint \eqref{offset_eq} is applied. 
\begin{fact}\label{diff_maps_params}
	Under the offset constraint \eqref{offset_eq}, CDM parameter differences $u_{xy} - u_{yx}$ and $u_{xz} - u_{yz}$, $\forall x \neq y \neq z$, have a 1-to-1 mapping with CDM parameters $u_{xy}$ $\forall x \neq y$.
\end{fact}
\begin{proof} It is immediately obvious that given the parameters, we can uniquely construct the differences. For the other direction, consider that
\begin{align*}
u_{xy} &= u_{xy} + \frac{1}{n-1} \log\Big(\sum_{w \in \exx} \exp\Big(\sum_{z \in \exx \setminus z} u_{wz} \Big)\Big) \\
&= \frac{1}{n-1} \log\Big(\sum_{w \in \exx} \exp\Big(\sum_{z \in \exx \setminus w} u_{wz} - u_{xy} \Big)\Big) \\
&= \frac{1}{n-1} \log\Big(\sum_{w \in \exx} \exp\Big([u_{wy} - u_{xy}]\textbf{1}(w \neq y) + \sum_{z \in \exx \setminus {w,y}} u_{wz} - u_{xy} \Big)\Big) \\
&= \frac{1}{n-1} \log\Big(\sum_{w \in \exx} \exp\Big([u_{wy} - u_{xy}]\textbf{1}(w \neq y) + \sum_{z \in \exx \setminus {w,y}} [u_{zy} - u_{xy}] + [u_{yz} - u_{zy}] + [u_{wz} - u_{yz}] \Big)\Big).
\end{align*}
Here the first equality follows because the second term on the right hand size is 0, by the offset constraint \eqref{offset_eq}. The remaining equalities are simply algebraic manipulations. The last equality is purely a function of differences following the aforementioned statement, therefore proving the claim. 
\end{proof}
Hence, statements regarding identifiability between CDM parameter differences of the pattern $u_{xy} - u_{yx}$ and $u_{xz} - u_{yz}$ $\forall x \neq y \neq z$ and the $\beta$'s can be mapped to statements about identifiability between CDM parameters and probabilities.  

We now link the above facts with the following: the $\beta$'s can be conveniently represented in terms of these CDM parameter differences. Using $u\in \mathbb{R}^{n(n-1)}$ to refer to a vectorization of the parameters, with elements of the vector indexed as we have so far (i.e.,~$u_{xy}$ finds the subset of $(n-1)$ entries associated with item $x$, and finds the contextual role of item $y$ within those entries), we have the following fact.

\begin{fact}\label{beta_param_diff}
For any set $C$ and any $x \in C$,
	$\beta_{x,C} = \frac{1}{|C|}\sum_{y \in C \setminus x}\Big([u_{xy}-u_{yx}] + \sum_{z \in C \setminus \{x,y\}}[u_{xz} - u_{yz}]\Big).$
\end{fact}
\begin{proof}
From the definition of $\beta_{x,C}$ in equation \eqref{eqn:betaxC} we have:
\begin{align*}
    \beta_{x,C} &= \log(\frac{P_{x,C}}{\bar{P_{C}}})\\ 
    &= \sum_{z \in C \setminus x} u_{xz} - \frac{1}{|C|}\sum_{y \in C} \sum_{z \in C \setminus y} u_{yz}\\ 
    &= \frac{1}{|C|}\sum_{y \in C \setminus x}\Big([u_{xy}-u_{yx}] + \sum_{z \in C \setminus \{x,y\}}[u_{xz} - u_{yz}]\Big)\\
\end{align*}
Here the final equality is a rearrangement of terms into the parameter differences of interest.  
\end{proof}
We introduce an indicator vector $g_{x,C} \in \mathbb{Z}^{n(n-1)}$ that contains non-zero values at the relevant indices of $u$ so that the final equality can be rewritten as
\begin{align}\label{indicator_g_def}
    \frac{1}{|C|}\sum_{y \in C \setminus x}\Big([u_{xy}-u_{yx}] + \sum_{z \in C \setminus \{x,y\}}[u_{xz} - u_{yz}]\Big) = \frac{1}{|C|} g_{x,C}^T u.
\end{align}
Lastly, we state and prove the following lemma, which will serve as the departure point for the three proofs.
Consider a collection $\uniqueD$ of unique subsets of the universe $\mathcal{X}$ of sizes 2 or greater, and let $\Omega = \sum_{C \in \uniqueD} |C|$ be the sum of the sizes of all the sets.
We then refer to a system design matrix $G(\uniqueD) \in \mathbb{Z}^{\Omega \times n(n-1)}$ as the linear system relating the parameters $u$ to the scaled log probability ratios $|C|\beta_{x,C}$. We construct such a matrix by concatenating, for each set $C \in \uniqueD$, for every item $x \in C$, the indicator vector $g_{x,C}^T$, as defined in \eqref{indicator_g_def}, as a row.
\begin{lemma}\label{lemma:rank_condition}
The full rank CDM is identifiable up to a shift for collection $\uniqueD$ iff $\text{rank}(G(\uniqueD)) = n(n-1) - 1$.
\end{lemma}
\begin{proof}
Clearly, $\text{rank}(G(\uniqueD)) \leq n(n-1)-1$, due to the shift invariance of $u$. That is, $G$ is only specified in terms of differences of elements in $u$, and hence $\text{null}(G(\uniqueD)) \ni \textbf{1}$.  

Suppose first that $\text{rank}(G(\uniqueD)) = n(n-1) - 1$. Then, for two vectors $u_1, u_2 \in \mathbb{R}^{n(n-1)}$, if $u_1 \neq \alpha \textbf{1} + u_2$ for any $\alpha \in \mathbb{R}$ then $\beta_1 = \mathbf{C}^{-1}Gu_1 \neq Gu_2 = \mathbf{C}^{-1}\beta_2$, where $\mathbf{C}^{-1} \in \mathbb{R}^{\Omega \times \Omega}$ is the diagonal matrix with values are $\frac{1}{|C|}, \forall C \in \uniqueD$ (which undoes the scaling of the scaled log probability ratios). Since Fact $\ref{unique_beta_map_prob}$ states that $\beta$'s have a unique mapping with the choice system probabilities over the collection $\uniqueD$, $u$ vectors are identifiable up to a shift for a given set of probabilities over the collection $\uniqueD$. 

Suppose now that $\text{rank}(G(\uniqueD)) < n(n-1) - 1$. Then, there exists some vector $v \in \text{null}(G(\uniqueD)), v \neq \alpha \textbf{1}$ for any $\alpha$, for which $\mathbf{C}^{-1}G(\uniqueD)(u_1) = \mathbf{C}^{-1}G(\uniqueD)(u_1 + v)$. Again since the $\beta$'s uniquely map to the probabilities, there exist two $u$ vectors different beyond a shift that map to the same set of choice system probabilities. Hence, $u$ is not identifiable up to a shift.
\end{proof}

We add as an additional note that under the offset constraint \eqref{offset_eq}, the CDM parameters are uniquely identifiable, following the analysis of Fact \ref{diff_maps_params}. Now we proceed to proving the individual theorems, each of which essentially boils down to analyzing the rank of the system design matrix $G(\uniqueD)$ of collections $\uniqueD$ comprised of sets of a single size, of collections $\uniqueD$ comprised of sets of multiple sizes, and formalizing the calculation of $G(\uniqueD)$ for a given dataset.

\subsection{Proof of Theorem \ref{thm:two_sets_good}}

{\bf Proof.}
It is sufficient to show that the statement holds for the full rank case, as further constraining the parameters using rank conditions does not affect identifiability. Note that the statement of the theorem is a sufficient condition for identifiability, and for low-rank CDMs in particular it is possibly an overly strong requirement.

Consider two different subset sizes $s$ and $t$, and assume wlog that $t$ is within $[3,n-1]$. For any $\{x,y\}$, consider $C_{wz} \ni \{x,y\}$, $|C_{wz}| = t-1$, indexed by items $\{w,z\} \in \mathcal X$, $\{w,z\} \notin C_{wz}$. Let $A_{wz} = C_{wz} \cup \{w\}$ and $B_{wz} = C_{wz} \cup \{z\}$. Using $\beta_{xy}^C$ as shorthand for  $\beta_{x,C} - \beta_{y,C}$., we have that
\begin{align*}
\beta_{xy}^{A_{wz}} - \beta_{xy}^{B_{wz}} = [u_{xw} - u_{yw}] - [u_{xz} - u_{yz}].
\end{align*}

Now, if $s < t$, Take $D \ni \{x,y\}$ of size $s$ and $A$ (of size $t$) such that $D \subset A$. Now,
\begin{align*}
\beta_{xy}^{A} - \beta_{xy}^{D} = \sum_{q \in A\setminus D} [u_{xq} - u_{yq}].
\end{align*}
Then, we can solve for $[u_{xw} - u_{yw}]$ as follows:
\begin{align*}
[u_{xw} - u_{yw}] = \frac{1}{t-s}(\beta_{xy}^{A} - \beta_{xy}^{D} + \sum_{q \in A\setminus D} \beta_{xy}^{A_{wq}} - \beta_{xy}^{B_{wq}}).
\end{align*}
With this relation we see that $[u_{xy} - u_{yx}] = \beta_{xy}^{A} - \sum_{q \in A\setminus \{x,y\}} [u_{xq} - u_{yq}]$.

If $s > t$, Take $D$ of size $s$ such that $A \subset D$. We then see that
$\beta_{xy}^{D} - \beta_{xy}^{A} = \sum_{q \in D\setminus A} [u_{xq} - u_{yq}]$, and as before, we can solve for $[u_{xw} - u_{yw}]$ as: 
\begin{align*}
[u_{xw} - u_{yw}] = \frac{1}{s-t}(\beta_{xy}^{D} - \beta_{xy}^{A} + \sum_{q \in D\setminus A} \beta_{xy}^{A_{wq}} - \beta_{xy}^{B_{wq}}).
\end{align*}
With this relation we see that $[u_{xy} - u_{yx}] = \beta_{xy}^{D} - \sum_{q \in D\setminus \{x,y\}} [u_{xq} - u_{yq}]$.

Applying Facts \ref{unique_beta_map_prob} and \ref{diff_maps_params}, statements regarding identifiability between CDM parameter differences of the pattern $u_{xy} - u_{yx}$ and $u_{xz} - u_{yz}$ $\forall x \neq y \neq z$ and the $\beta$'s can be mapped to statements about identifiability between CDM parameters and probabilities. We then conclude that the CDM parameters can be uniquely recovered from probabilities over two choice sets. Thus, comparisons over all choice sets of two sizes uniquely identify the CDM.
\qed

\subsection{Proof of Theorem \ref{thm:single_set_bad}}

{\bf Proof.}
To prove this claim, we separately consider three conditions on the set size s: $s = 2$, $s = n$, and $3 \leq s \leq n-1$. For each case, we first demonstrate the result for the full rank CDM and then show that every low rank CDM suffers from the same problem. 

In terms of notation, we consider a $U$ ``matrix'', $U \in \mathbb{R}^{n \times n}$, organizing the parameters $u_{xy}$, $\forall x \neq y$, with the matrix diagonal taking on arbitrary unused values. For the low rank case, the $U$ matrix is the dot product of the matrix of target vectors $T \in \mathbb{R}^{n \times r}$ and the matrix of context vector $C \in \mathbb{R}^{n \times r}$. Here, the diagonal formed by $t_x\cdot c_x$ can be arbitrary and is unused. We also use $\beta_{xy}^C$ as shorthand for  $\beta_{x,C} - \beta_{y,C}$.

\subsubsection*{{\it (i)} $s=2$}
For any pair $C = \{x,y\}$, $\beta_{xy}^C = u_{xy} - u_{yx}$. Thus, increasing both $u_{xy}$ and $u_{yx}$ by the same value leaves the pairwise probabilities unchanged. Thus the CDM parameter $U$ matrix is only specified up to a symmetric matrix $A$, where $U + A$ produces the same pairwise probabilities as $U$. 

Any rank $r$ matrix also suffers from the same identifiability issue: consider $T + B$ and $C + F$, where $B = \beta C + \gamma_1 \alpha \beta T$, and $F = \alpha T + \gamma_2 \alpha \beta C$ for $\alpha,\beta \in \mathbb{R}$, $\gamma_1, \gamma_2 \in \{0,1\}, \gamma_1 \neq \gamma_2$. These scalar parameters form a subset of perturbations that modify the dot product $U = TC^T$ only by a symmetric matrix, thereby leaving the pairwise probabilities unchanged. 

\subsubsection*{{\it (ii)} $s=n$}
For the full universe $\exx$, $\beta_{xy}^\exx = u_{xy} - u_{yx} + \sum_{z \in \exx \setminus \{x,y\}} u_{xy} - u_{yx}$. Consider then any matrix $A \in \mathbb{R}^{n\times n}$ that has $(A - \text{diag(A)})\textbf{1} = g\textbf{1}$, where $g$ is a constant and $\textbf{1} \in \mathbb{R}^{n \times 1}$ is the vector of all ones. This is, any matrix $A$ where the rows (not including the diagonal) all sum to the same constant.  Then $U$ and $U+A$ have the same choice probabilities on the full universe set. 

For the identifiability problem to transfer to the rank $r$ case, we find $T + \gamma_1 B$ and $C + \gamma_2 F$ where $\gamma_1, \gamma_2 \in \{0,1\}, \gamma_1 \neq \gamma_2$ such that the perturbation to a $U$ matrix follows the same properties as the matrix $A$ in the full rank case above. We show how to find such a matrix for the rank $1$ case, which is sufficient for all rank $r$. Consider $U = tc^T$, where $t,c \in \mathbb{R}^{n \times 1}$. We may perturb $t$ by a vector $b \in \mathbb{R}^{n \times 1}$ where $b_x = \frac{g}{(c^T\textbf{1} - c_x)}$, $\forall x$, for any constant $g$, as long as $(c^T\textbf{1} - c_x) \neq 0$ $\forall x$. In case $(c^T\textbf{1} - c_y) = 0$ for any $y$, set $g = 0$, $b_x = 0 \ \forall x \neq y$, and $b_y$ to any arbitrary value. The perturbation to $U$ is then $bc^T$, and we leave the reader to verify $((bc^T) - \text{diag}(bc^T))\textbf{1} = g\textbf{1}$, thereby not changing the universe probabilities. Similarly, we may perturb $c$ by a vector $f$, where $f_x = g[\frac{1}{n-1} \sum_{z} (\frac{1}{t_z}) - \frac{1}{t_x}]$ if $t_x \neq 0$, $\forall x$. In case $t_y = 0$ for some $y$, set $g = 0$, $f_x = 0$, $\forall x \neq y$, and $f_y$ to any arbitrary value. The perturbation to $U$ is then $t^Tf$, and we have $((tf^T) - \text{diag}(tf^T))\textbf{1} = g\textbf{1}$, thereby not changing the universe probabilities.
\subsubsection*{{\it (iii)} $3 \leq s \leq n-1$}
For all other set sizes, we again show the identifiability issue for the full rank case, and show that the null space in the parameters also transfers over to the rank $r$ case. Consider any $C \ni \{x,y\}$, $\{w,z\} \notin C$ of size $s-1$ for any $\{x,y,w,z\}$. Take $C_w = C \cup \{w\}$, and $C_z = C \cup \{z\}$. Note that we can always identify such sets because we are in the size regime $3 \leq s \leq n-1$. Then, $\beta_{xy}^{C_w} - \beta_{xy}^{C_z} = [u_{xw} - u_{yw}] - [u_{xz} - u_{yz}]$. Thus, given $[u_{xz} - u_{yz}]$ for a single $z$, we can set $[u_{xw} - u_{yw}] = \beta_{xy}^{C_w} - \beta_{xy}^{C_z} + [u_{xz} - u_{yz}]$, and set $[u_{xy} - u_{yx}] = \beta_{xy}^{C_z} - \sum_{q \in {C_z}\setminus \{x,y\}}[u_{xq} - u_{yq}] = \beta_{xy}^{C_z} - \sum_{q \in {C_z}\setminus \{x,y\}}[\beta_{xy}^{C_z} - \beta_{xy}^{C_q}] - (s-2)[u_{xz} - u_{yz}]$ to keep the choice probabilities unchanged.
This invariance implies that the $U$ matrix can be perturbed by the rank-$1$ matrix $a\textbf{1}^T$ where $a \in \mathbb{R}^{n \times 1}$ is any vector and the choice probabilities are unchanged. 

We can now show that such perturbations to $U$ can be produced in the rank $r$ case by modifiying $C$. Consider $C + \textbf{1}b^T$ where $b \in \mathbb{R}^{r \times 1}$. Then, $U = T(C + \textbf{1}b^T)^T = TC^T + (Tb)\textbf{1}^T$, which is a perturbation to $U$ of the proper form.
Through these three cases, we have now shown that every rank $r$ CDM cannot be uniquely identified even when provided all comparisons of a single choice set size. 
\qed

\subsection{Proof of Theorem \ref{thm:rank_test}}

{\bf Proof.}
Consider a dataset of the form $\mathcal{D} = \lbrace (x_j, C_j) \rbrace_{j=1}^m$ of a decision maker making choices: a datapoint $j$ represents a decision scenario, and contains $C_j$, the context provided in that decision, and $x_j \in C_j$, the item chosen in the context. 
Recall that $\Omega_\mathcal{D} = \sum_{j=1}^m |C_j|$. Construct then a matrix $G(\mathcal{D}) \in \mathbb{Z}^{\Omega_\mathcal{D} \times n(n-1)}$ by concatenating, for every datapoint $j$, for every item $x \in C_j$, the indicator vector $g_{x,C_j}^T$ as defined in equation \eqref{indicator_g_def} as a row.
Denoting $\uniqueD$ as the collection of unique choice sets in dataset $\mathcal{D}$, it is clear that $\text{rank}{(G(\mathcal{D}))} = \text{rank}{(G(\uniqueD))}$, where the latter matrix is defined as in Lemma \ref{lemma:rank_condition} for the collection $\uniqueD$. This equality of ranks follows from the fact that the set of unique rows of $G(\mathcal{D})$ are the same as those in $G(\uniqueD)$, and repeated rows do not change the rank of a matrix. Thus, we can directly test whether a dataset results in an identifiable CDM by testing the rank of $G(\mathcal{D})$.
\qed


\clearpage

\section{Convergence proof} 
\label{app:convergence}

We restate and then prove Theorem~\ref{thm:CDM_bound}.

{\bf Theorem~\ref{thm:CDM_bound}.}
\emph{
Let $u^\star$ denote the true CDM model from which data is drawn. Let $\hat{u}_{\text{MLE}}$ denote the maximum likelihood solution. Assume $\uniqueD$ identifies the CDM. For any $u^\star \in \mathcal{U_B} = \lbrace u \in \mathbb{R}^d : \left\lVert u\right\rVert_\infty \leq B, \textbf{1}^Tu = 0 \rbrace$, and expectation taken over the dataset $\mathcal{D}$ generated by the CDM model,
$$
\mathbb{E}\big[\left\lVert 
\hat{u}_{\text{MLE}}(\mathcal D) - u^\star \right\rVert_2^2\big] 
\leq 
c_{B,k_{\text{max}}}\frac{d-1}{m},
$$
where $k_{\text{max}}$ refers to the maximum choice set size in the dataset, and $c_{B,k_{\text{max}}}$ is a constant that depends on $B$, $k_{\text{max}}$ and the spectrum of the design matrix $G(\mathcal{D})$. 
}

{\bf Proof.}
We describe the sampling process as follows using the same notation as before. Given some true CDM $u^\star \in \mathcal{U}_B$, for each datapoint $j \in [m]$ we have the probability of choosing item $x$ from set $C_j$ as 
\begin{align*}
    \mathbb{P}(y_j = x | u^\star, C_j) = \frac{\exp(\sum_{z \in C_j \setminus x} u^\star_{x z})}{\sum_{y \in C_j} \exp(\sum_{z \in C_j \setminus y} u^\star_{yz}))}.
\end{align*}

We now introduce notation that will let us represent the above expression in a more compact manner. Because our datasets involve choice sets of multiple sizes, we use $k_j \in [k_{\text{min}}, k_{\text{max}}]$ to denote the choice set size for datapoint $j$. Extending a similar concept in \cite{shah2016estimation} to the multiple set sizes, and the more complex structure of the CDM, we then define matrices $E_{j,k_j} \in \mathbb{R}^{d \times k_j}, \ \forall j \in [m]$ as follows: $E_{j,k_j}$ has a column for every item $y \in C_j$ (and hence $k_j$ columns), and the column corresponding to item $y \in C_j$ has a one at the position of each $u_{yz}$ for $z \in C_j \setminus y$, and zero otherwise. This construction allows us to write the familiar expressions $\sum_{z \in C_j \setminus y} u_{yz}$, for each $y$, simply as a single vector-matrix product $u^T E_{j,k_j} = [\sum_{z \in C_j \setminus y_1} u_{y_1z}, \sum_{z \in C_j \setminus y_2} u_{y_2z}, \ldots \sum_{z \in C_j \setminus y_{k_j}} u_{y_{k_j}z}] \in \mathbb{R}^{1 \times k_j}$. 

Next, we define a collection of functions $F_k : \mathbb{R}^{k} \mapsto [0,1]$, $\forall k \in [k_{\text{min}}, k_{\text{max}}]$ as 
\begin{align*}
F_k([x_1,x_2,\ldots,x_k]) = \frac{\exp(x_1)}{\sum_{l=1}^k \exp(x_l)},
\end{align*}
where the numerator always corresponds to the first entry of the input. These functions $F_k$ have several properties that will become useful later in the proof. First, it is easy to verify that all $F_k$ are shift-invariant, that is, $F_k(x) = F_k(x + c\mathbf{1})$, for any scalar $c$. Next, we show that all $F_k$ are strongly log-concave, that is, $\nabla^2(-\log(F_k(x))) \succeq H_k$ for some $H_k \in \mathbb{R}^{k \times k}$, $\lambda_2(H_k) > 0$. The proof for this property stems directly from its counterpart in \cite{shah2016estimation}, as multiple set sizes does not affect the result. We compute the Hessian as:
\begin{align*}
\nabla^2(-\log(F_k(x))) = \frac{\exp(x_1)}{(\langle \exp(x), 1 \rangle)^4}(\langle \exp(x), 1 \rangle \text{diag}(\exp(x)) - \exp(x)\exp(x)^T),
\end{align*}
where $\exp(x) = [e^{x_1}, \ldots, e^{x_k}]$. Note that
\begin{align*}
v^T\nabla^2(-\log(F_k(x)))v &= \frac{\exp(x_1)}{(\langle \exp(x), 1 \rangle)^4}v^T(\langle \exp(x), 1 \rangle \text{diag}(\exp(x)) - \exp(x)\exp(x)^T)v\\ 
&= \frac{\exp(x_1)}{(\langle \exp(x), 1 \rangle)^4}(\langle \exp(x), 1 \rangle \langle \exp(x), v^2 \rangle - \langle \exp(x), v \rangle^2) \\
&\geq 0,
\end{align*}
where $v^2$ refers to the element-wise square operation on vector $v$. While the final inequality is an expected consequence of the positive semidefiniteness of the Hessian, we note that it also follows from an application of Cauchy-Schwarz to the vectors $\sqrt{\exp(x)}$ and $\sqrt{\exp(x)} \odot v$, and is thus an equality \textit{if and only if} $v \in \text{span}(\mathbf{1})$. Thus, we have that the smallest eigenvalue $\lambda_1(\nabla^2(-\log(F_k(x)))) = 0$ is associated with the vector $\textbf{1}$, a property we expect from shift invariance, and that the second smallest eigenvalue $\lambda_2(\nabla^2(-\log(F_k(x)))) > 0$. 
Thus, we can state that 
\begin{align}\label{F_str_cvx}
    \nabla^2(-\log(F_k(x))) \succeq H_k = \beta_k(I - \frac{1}{k}\mathbf{1}\mathbf{1}^T),
\end{align}
where
\begin{align}\label{eqn:beta}
\beta_k \defeq \min_{x \in [-(k-1)B, (k-1)B]^k} \lambda_2(\nabla^2(-\log(F_k(x)))),
\end{align}
and it's clear that $\beta_k>0$. The minimization is taken over $x \in [-(k-1)B, (k-1)B]^k$ since each $x_i$ is a sum of $k-1$ values of the $u$ vector, each entry of which is in $[-B,B]$. We conclude that all $F_k$ are strongly log-concave.

As a final notational addition, in the same manner as \cite{shah2016estimation} but accounting for multiple set sizes, we define $k$ permutation matrices $R_{1,k},\ldots,R_{k,k} \in \mathbb{R}^{k, k}, \forall k \in [k_{\text{min}}, k_{\text{max}}]$, representing $k$ cyclic shifts in a fixed direction. That is, these matrices allow for the cycling of the entries of row vector $v \in \mathbb{R}^{1 \times k}$ so that any entry can become the first entry of the vector, for any of the relevant $k$. This construction allows us to represent any choice made from the choice set $C_j$ as the first element of the vector $x$ that is input to $F$, thereby placing it in the numerator.

Given the notation introduced above, we can now state the probability of choosing the item $x$ from set $C_j$ compactly as:
\begin{align*}
    \mathbb{P}(y_j = x | u^\star, C_j) = \mathbb{P}(y_j = x |u^\star, k_j, E_{j,k_j}) = F_{k_j}({u^\star}^T E_{j,k_j} R_{x,k_j}).
\end{align*}
We can then rewrite the full-rank CDM likelihood as
\begin{align*}
\sup_{u \in \mathcal{U}_B} \prod_{(x_j, k_j, E_{j,k_j}) \in \mathcal{D}} F_{k_j}(u^T E_{j,k_j} R_{x_j,k_j}),
\end{align*}
and the scaled negative log-likelihood as
\begin{align*}
\ell(u) = -\frac{1}{m}\sum_{(x_j, k_j, E_{j,k_j}) \in \mathcal{D}} \log(F_{k_j}(u^T E_{j,k_j} R_{x_j,k_j})) = -\frac{1}{m}\sum_{j=1}^m \sum_{i=1}^{k_j} \mathbf{1}[y_j = i] \log(F_{k_j}(u^T E_{j,k_j} R_{i,k_j})).
\end{align*}
Thus, 
\begin{align*}
    \hat{u}_\text{MLE} = \arg \max_{u \in \mathcal{U}_B} \ell(u).
\end{align*}

The compact notation makes the remainder of the proof a straightforward application of results from convex analysis: we first demonstrate that the scaled negative log-likelihood is strongly convex with respect to a semi-norm\footnote{A semi-norm is a norm that allows non-zero vectors to have zero norm.}, and we use this property to show the proximity of the MLE to the optimal point as desired. The remainder of the proof exactly mirrors that in \cite{shah2016estimation} with a few extra steps of accounting created by the multiple set sizes. The notable exception is in the definition of $L$, and conditions about its eigenvalues that tie back to the previous results about identifiability. While in \cite{shah2016estimation} there is a clear connection of $L$ to the graph Laplacian matrix of the item comparison graph, it is unclear here how to interpret $L$ as a graph Laplacian. 

First, we have the gradient of the negative log-likelihood as
\begin{align*}
\nabla \ell(u) = -\frac{1}{m}\sum_{j=1}^m \sum_{i=1}^{k_j} \mathbf{1}[y_j = i] E_{j,k_j} R_{i,k_j}\nabla\log(F_{k_j}(u^T E_{j,k_j} R_{i,k_j})),
\end{align*}
and the Hessian as
\begin{align*}
\nabla^2 \ell(u) = -\frac{1}{m}\sum_{j=1}^m \sum_{i=1}^{k_j} \mathbf{1}[y_j = i] E_{j,k_j} R_{i,k_j}\nabla^2\log(F_{k_j}(u^T E_{j,k_j} R_{i,k_j})) R_{i,k_j}^TE_{j,k_j}^T.
\end{align*}
We then have, for any vector $z \in \mathbb{R}^d$,
\begin{align*}
z^T\nabla^2 \ell(u)z &= -\frac{1}{m}\sum_{j=1}^m \sum_{i=1}^{k_j} \mathbf{1}[y_j = i] z^TE_{j,k_j} R_{i,k_j}\nabla^2\log(F_{k_j}(u^T E_{j,k_j} R_{i,k_j})) R_{i,k_j}^TE_{j,k_j}^Tz \\
&=\frac{1}{m}\sum_{j=1}^m \sum_{i=1}^{k_j} \mathbf{1}[y_j = i] z^TE_{j,k_j} R_{i,k_j}\nabla^2(-\log(F_{k_j}(u^T E_{j,k_j} R_{i,k_j}))) R_{i,k_j}^TE_{j,k_j}^Tz \\
&\geq \frac{1}{m}\sum_{j=1}^m \sum_{i=1}^{k_j} \mathbf{1}[y_j = i] z^TE_{j,k_j} R_{i,k_j}H_k R_{i,k_j}^TE_{j,k_j}^Tz \\
&= \frac{1}{m}\sum_{j=1}^m \sum_{i=1}^{k_j} \mathbf{1}[y_j = i] z^TE_{j,k_j} R_{i,k_j}\beta_{k_j}(I - \frac{1}{k_j}\mathbf{1}\mathbf{1}^T) R_{i,k_j}^TE_{j,k_j}^Tz \\
&\geq \beta_{k_\text{max}}\frac{1}{m}\sum_{j=1}^m \sum_{i=1}^{k_j} \mathbf{1}[y_j = i] z^TE_{j,k_j} (I - \frac{1}{k_j}\mathbf{1}\mathbf{1}^T)E_{j,k_j}^Tz \\
&=\beta_{k_\text{max}}\frac{1}{m}\sum_{j=1}^m z^TE_{j,k_j} (I - \frac{1}{k_j}\mathbf{1}\mathbf{1}^T)E_{j,k_j}^Tz.
\end{align*}
The first line follows from applying the definition of the Hessian. The second line follows from pulling the negative sign into the $\nabla^2$ term. The third and fourth line follow from \eqref{F_str_cvx}, strong log-concavity of all $F_k$. The fifth line follows from the pulling out $\beta_{k_j}$ and lower bounding it with $\beta_{k_\text{max}}$ and recognizing that $H_k$ is invariant to permutation matrices. The sixth line follows from removing the inner sum since the terms are independent of $i$.
Now, defining the matrix $L$ as 
\begin{align*}
    L = \frac{1}{m}\sum_{j=1}^m E_{j,k_j} (I - \frac{1}{k_j}\mathbf{1}\mathbf{1}^T)E_{j,k_j}^T,
\end{align*}
we first note a few properties of $L$. First, it is easy to verify that $L\mathbf{1} = 0$, and hence $\text{span}(\mathbf{1}) \subseteq \text{null}(L)$. Moreover, we now show that $\lambda_2(L) > 0$, that is, $\text{null}(L) \subseteq \text{span}(\mathbf{1})$. Consider the matrix $G(\mathcal{D})$ in Theorem~\ref{thm:rank_test}. Define a matrix $X(\mathcal{D}) = \mathbf{C}_\mathcal{D}^{-1}G(\mathcal{D})$, where $\mathbf{C}_\mathcal{D}^{-1} \in \mathbb{R}^{\Omega_\mathcal{D} \times \Omega_\mathcal{D}}$ is the diagonal matrix with values are $\frac{1}{k_j}$, for every datapoint $j$, for every item $x \in C_j$. Simple calculations show that, $$L = \frac{1}{m}X(\mathcal{D})^TX(\mathcal{D}) \succeq 0.$$ As a consequence of the properties of matrix rank, we then have that $\text{rank}(L) = \text{rank}(X(\mathcal{D})) = \text{rank}(G(\mathcal{D}))$. Thus, from Theorem \ref{thm:rank_test}, we have that if the dataset $\mathcal{D}$ identifies the CDM, $\text{rank}(L) = d-1$, and hence $\lambda_2(L) > 0$. With this matrix, we can write,
\begin{align*}
    z^T\nabla^2 \ell(u)z \geq \beta_{k_\text{max}}z^TLz = \beta_{k_\text{max}}||z||_L^2,
\end{align*}
which is equivalent to stating that $\ell(u)$ is $\beta_{k_\text{max}}$-strongly convex with respect to the $L$ semi-norm at all $u \in \mathcal{U}_B$. Since $u^\star, \hat{u}_{\text{MLE}} \in \mathcal{U}_B$, strong convexity implies that
\begin{align*}
    \beta_{k_\text{max}}||\hat{u}_{\text{MLE}} - u_{\star}||_L^2 \leq \ell(\hat{u}_{\text{MLE}}) - \ell(u^\star) - \langle \nabla \ell(u^\star), \hat{u}_{\text{MLE}} - u^\star \rangle.
\end{align*}
Further, we have
\begin{align*}
    \ell(\hat{u}_{\text{MLE}}) - \ell(u^\star) - \langle \nabla \ell(u^\star), \hat{u}_{\text{MLE}} - u^\star \rangle &\leq - \langle \nabla \ell(u^\star), \hat{u}_{\text{MLE}} - u^\star \rangle \\
    &\leq |(\hat{u}_{\text{MLE}} - u^\star)^T\nabla \ell(u^\star)| \\
    &= |(\hat{u}_{\text{MLE}} - u^\star)^T L^{\frac{1}{2}}{L^{\frac{1}{2}}}^\dagger \nabla \ell(u^\star)| \\
    &\leq ||L^{\frac{1}{2}}(\hat{u}_{\text{MLE}} - u^\star)||_2||{L^{\frac{1}{2}}}^\dagger\nabla \ell(u^\star)||_2 \\&
    = ||\hat{u}_{\text{MLE}} - u^\star||_L||\nabla \ell(u^\star)||_{L^\dagger}.
\end{align*}
Here the third line follows from the fact that $\textbf{1}^T(\hat{u}_{\text{MLE}} - u^\star) = 0$, and so $(\hat{u}_{\text{MLE}} - u^\star) \perp \text{null}(L)$, which also implies that $(\hat{u}_{\text{MLE}} - u^\star) \perp \text{null}(L^{\frac{1}{2}})$, and so $(\hat{u}_{\text{MLE}} - u^\star)L^{\frac{1}{2}}{L^{\frac{1}{2}}}^\dagger =(\hat{u}_{\text{MLE}} - u^\star)$. The fourth line follows from Cauchy-Schwarz. Thus, we can conclude that 
\begin{align*}
    ||\hat{u}_{\text{MLE}} - u^\star||_L^2 \leq \frac{1}{\beta_{k_\text{max}}^2}||\nabla \ell(u^\star)||_{L^\dagger}^2 = \frac{1}{\beta_{k_\text{max}}^2}\nabla \ell(u^\star)^TL^\dagger\nabla \ell(u^\star).
\end{align*}
Now, all that remains is bounding the term on the right hand side. Recall the expression for the gradient
\begin{align*}
\nabla \ell(u^\star) = -\frac{1}{m}\sum_{j=1}^m \sum_{i=1}^{k_j} \mathbf{1}[y_j = i] E_{j,k_j} R_{i,k_j}\nabla\log(F_{k_j}({u^\star}^T E_{j,k_j} R_{i,k_j})) = -\frac{1}{m}\sum_{j=1}^m E_{j,k_j}V_{j,k_j},
\end{align*}
where in the equality we have defined $V_{j,k_j} \in \mathbb{R}^{k_j}$ as $$
V_{j,k_j} \defeq \sum_{i=1}^{k_j} \mathbf{1}[y_j = i] R_{i,k_j}\nabla\log(F_{k_j}({u^\star}^T E_{j,k_j} R_{i,k_j})).
$$
Useful in our analysis will be an alternate expression for the gradient, 
\begin{align*}
\nabla \ell(u^\star) = -\frac{1}{m}\sum_{j=1}^m E_{j,k_j}V_{j,k_j} = -\frac{1}{m}X(\mathcal{D})^TV,
\end{align*}
where we have defined $V \in \mathbb{R}^{\Omega_\mathcal{D}}$ as the concatenation of all $V_{j,k_j}$.

Now, we have 
\begin{align}\label{F_grad}
    (\nabla\log(F_{k}(x)))_l =  \mathbf{1}[l=1] - \frac{\exp(x_l)}{\sum_{p=1}^k \exp(x_p)},
\end{align}
and so $\langle \nabla\log(F_{k}(x)), \mathbf{1} \rangle = \frac{1}{F_{k}(x)}\langle \nabla F_{k}(x), \mathbf{1} \rangle = \sum_{l=1}^k(\mathbf{1}[l=1] - \frac{\exp(x_l)}{\sum_{p=1}^k \exp(x_p)}) = 0$, and hence, $V_{j,k_j}^T\textbf{1} = 0$. 

We now consider the matrix $M_k = (I - \frac{1}{k}\mathbf{1}\mathbf{1}^T)$. We note that $M_k$ has rank $k-1$, with its nullspace corresponding to the span of the ones vector. We state the following identities:
\begin{align*}
    M_k = M_k^\dagger = M_k^{\frac{1}{2}} = {M_k^\dagger}^{\frac{1}{2}}.
\end{align*}
Thus we have $M_{k_j}V_{j,k_j} = {M_{k_j}}^{\frac{1}{2}}M_{k_j}^{\frac{1}{2}}V_{j,k_j} = M_k M_k^\dagger V_{j,k_j} = V_{j,k_j}$, where the last equality follows since $V_{j,k_j}$ is orthogonal to the nullspace of $M_{k_j}$. Now, taking expectations over the dataset, we have,
\begin{align*}
\mathbb{E}[V_{j,k_j}] &= \mathbb{E}\Big[\sum_{i=1}^{k_j} \mathbf{1}[y_j = i] R_{i,k_j}\nabla\log(F_{k_j}({u^\star}^T E_{j,k_j} R_{i,k_j}))\Big]\\
&=\sum_{i=1}^{k_j} \mathbb{E}\Big[\mathbf{1}[y_j = i]\Big] R_{i,k_j}\nabla\log(F_{k_j}({u^\star}^T E_{j,k_j} R_{i,k_j}))\\
&=\sum_{i=1}^{k_j} F_{k_j}({u^\star}^T E_{j,k_j} R_{i,k_j}) R_{i,k_j}\nabla\log(F_{k_j}({u^\star}^T E_{j,k_j} R_{i,k_j}))\\
&=\sum_{i=1}^{k_j} R_{i,k_j}\nabla F_{k_j}({u^\star}^T E_{j,k_j} R_{i,k_j})\\
&=\nabla_z \Big(\sum_{i=1}^{k_j} F_{k_j}(z^T R_{i,k_j})\Big) = \nabla_z (1) = 0.
\end{align*}
Here, the third equality follows from applying the expectation to the indicator and retrieving the true probability. The fourth line follows from applying the definition of gradient of log, and the final line from performing a change of variables $z = {u^\star}^TE_{j,k_j}$, pulling out the gradient and undoing the chain rule, and finally, recognizing that the expression sums to $1$ for any $z$, thus resulting in a $0$ gradient. We note that an immediate consequence of the above result is that $\mathbb{E}[V] = 0$, since $V$ is simply a concatenation of the individual $V_{j,k_j}$.

Next, we have
\begin{align*}
\mathbb{E}[\nabla\ell(u^\star)^TL^\dagger\nabla \ell(u^\star)] &= \frac{1}{m^2}\mathbb{E}\Big[\sum_{j=1}^{m}\sum_{l=1}^{m} V_{j,k_j}^TE_{j,k_j}^TL^\dagger E_{l,k_l}V_{l,k_l}\Big]\\
&= \frac{1}{m^2}\mathbb{E}\Big[\sum_{j=1}^{m}\sum_{l=1}^{m} V_{j,k_j}^T{M_{k_j}}^{\frac{1}{2}}E_{j,k_j}^TL^\dagger E_{l,k_l}{M_{k_l}}^{\frac{1}{2}}V_{l,k_l}\Big]\\
&=\frac{1}{m^2}\mathbb{E}\Big[\sum_{j=1}^{m} V_{j,k_j}^T{M_{k_j}}^{\frac{1}{2}}E_{j,k_j}^TL^\dagger E_{j,k_j}{M_{k_j}}^{\frac{1}{2}}V_{j,k_j}\Big]\\
&\leq\frac{1}{m}\mathbb{E}\Big[\sup_{l \in [m]}||V_{l,k_l}||_2^2\Big]\frac{1}{m}\sum_{j=1}^{m} \textbf{tr}\Big({M_{k_j}}^{\frac{1}{2}}E_{j,k_j}^TL^\dagger E_{j,k_j}{M_{k_j}}^{\frac{1}{2}}\Big)\\
&=\frac{1}{m}\mathbb{E}\Big[\sup_{l \in [m]}||V_{l,k_l}||_2^2\Big]\frac{1}{m}\sum_{j=1}^{m} \textbf{tr}\Big(L^\dagger E_{j,k_j}{M_{k_j}}^{\frac{1}{2}}{M_{k_j}}^{\frac{1}{2}}E_{j,k_j}^T\Big)\\
&= \frac{1}{m}\mathbb{E}\Big[\sup_{l \in [m]}||V_{l,k_l}||_2^2\Big]\textbf{tr}\Big(L^\dagger L\Big)\\
&=\frac{1}{m}\mathbb{E}\Big[\sup_{l \in [m]}||V_{l,k_l}||_2^2\Big](d-1),
\end{align*}
where the second line follows from identities of the $M$ matrix, the third from the independence of the $V_{j,k_j}$, the fourth from an upper bound of the quadratic form, the fifth from the properties of trace, the sixth from the definition of the matrix $L$, and the last from the value of the trace, which is simply the identity matrix with one zero entry in the diagonal. We then have that,
\begin{align*}
    \sup_{j \in [m]} ||V_{j,k_j}||_2^2 &= \sup_{j \in [m]}\sum_{i=1}^{k_j} \mathbf{1}[y_j = i] \nabla\log(F_{k_j}(u^T E_{j,k_j} R_{i,k_j}))^TR_{i,k_j}^TR_{i,k_j}\nabla\log(F_{k_j}(u^T E_{j,k_j} R_{i,k_j}))\\
    &=\sup_{j \in [m]} \sum_{i=1}^{k_j} \mathbf{1}[y_j = i] \nabla\log(F_{k_j}(u^T E_{j,k_j} R_{i,k_j}))^T\nabla\log(F_{k_j}(u^T E_{j,k_j} R_{i,k_j}))\\
    &=\sup_{j \in [m]} \sum_{i=1}^{k_j} \mathbf{1}[y_j = i] ||\nabla\log(F_{k_j}(u^T E_{j,k_j} R_{i,k_j}))||_2^2\\
    &\leq \sup_{v \in [-(k_\text{max} -1)B,(k_\text{max} -1)B]^{k_\text{max}}} ||\nabla \log(F_{k_\text{max}}(v))||_2^2 \leq 2,
\end{align*}
where $R_{i,k_j}^TR_{i,k_j}$ in the first line is simply the identity matrix. For the final line, recalling the expression for the log gradient of $F_k$ in equation \eqref{F_grad}, it is straightforward to show that $\sup_{v \in [-(k_\text{max} -1)B,(k_\text{max} -1)B]^{k_\text{max}}} ||\nabla \log(F_{k_\text{max}}(v))||_2^2$ is always upper bounded by 2. We again note that an immediate consequence of this is that the \textit{absolute value} of every element of $V$ is also upper bounded by 2.

Bringing this back to $\mathbb{E}[\nabla\ell(u^\star)^TL^\dagger\nabla \ell(u^\star)]$, we have that 
\begin{align*}
\mathbb{E}[\nabla\ell(u^\star)^TL^\dagger\nabla \ell(u^\star)] \leq \frac{2(d-1)}{m}.
\end{align*}
This immediately yields a bound on the expected risk in the $L$ semi-norm, which is,
\begin{align*}
    \mathbb{E}[||\hat{u}_{\text{MLE}} - u^\star||_L^2] \leq \frac{2(d-1)}{m\beta_{k_\text{max}}^2}.
\end{align*}
By noting that $||\hat{u}_{\text{MLE}} - u^\star||_L^2 = (\hat{u}_{\text{MLE}} - u^\star)L(\hat{u}_{\text{MLE}} - u^\star) \geq \lambda_2(L)||\hat{u}_{\text{MLE}} - u^\star||_L^2$, since $\hat{u}_{\text{MLE}} - u^\star \perp \text{null}(L)$, we can translate this into the $\ell_2$ norm:
\begin{align*}
    \mathbb{E}[||\hat{u}_{\text{MLE}} - u^\star||_2^2] \leq 
\frac{2(d-1)}{m\lambda_2(L)\beta_{k_\text{max}}^2}.
\end{align*}
Now, setting $$c_{B,k_{\text{max}}} \defeq \frac{2}{\lambda_2(L)\beta_{k_\text{max}}^2},$$ we retrieve the theorem statement,
$$
\mathbb{E}\big[\left\lVert 
\hat{u}_{\text{MLE}}(\mathcal D) - u^\star \right\rVert_2^2\big] 
\leq 
c_{B,k_{\text{max}}}\frac{d-1}{m}.
$$
We close with some remarks about $c_{B,k_{\text{max}}}$. The quantity $\beta_{k_\text{max}}$, defined in equation~\eqref{eqn:beta}, serves as the important term that approaches $0$ as a function of $B$ and $k_\text{max}$, requiring that the former be bounded. Finally, $\lambda_2(L)$ is a parallel to the requirements on the algebraic connectivity of the comparison graph in \cite{shah2016estimation} for the multinomial setting. Though the object $L$ here appears similar to the graph Laplacian $L$ in that work, there are major differences that are most worthy of further study.
\qed

\section{Auxiliary Material}
\subsection{Removing Constraints from $\mathcal{M}_2$}\label{section:constraints}
We restate $\mathcal{M}_2$ for convenience. 
\begin{align*}
P(x \mid C) = \frac{\exp(v(x) + \sum_{z \in C \setminus x} v(x \mid \{z\}) )}{\sum_{y \in C} \exp(v(y) + \sum_{z \in C \setminus y} v(y \mid \{z\}) )},
\\
\text{s.t. } \sum_{x \in \mathcal{X}} v(x) = 0, \ \ \sum_{x \in \mathcal{X} \setminus y} v(x \mid \{y\}) = 0, \ \ \forall y \in \mathcal{X}.
\end{align*}
Here, a counting exercise reveals that there are $n^2$ variables ($n$ from the $v(x)$ and $n(n-1)$ from the $v(x \mid \{u\})$ and there are $n+1$ linear equality constraints (1 from the constraint on $v(x)$, and $n$ from the constraints on $v(x \mid \{u\})$). Our goal in this step is to find a parameterization such that there remains only one equality constraint and $n(n-1)$ variables. To do this, we define the variable $u_{xz} \forall x \neq z \in \mathcal{X}$, and subject it to the constraint that $\sum_{x \in \mathcal{X}} \sum_{y \in \mathcal{X} \setminus x} u_{xy} = 0$. Set $v(z) = -\frac{1}{n-1}\sum_{x \in \mathcal{X} \setminus z} u_{xz}$, $\forall z$ and set $v(x \mid \{z\}) = u_{xz} - \frac{1}{n-1} \sum_{y \in \mathcal{X} \setminus z} u_{yz}$. We may then verify that $\sum_{z \in \mathcal{X}} v(z) = \frac{1}{n-1}\sum_{z \in \mathcal{X}}\sum_{x \in \mathcal{X} \setminus z} u_{xz} = 0$ because of the constraint on $u$. We can also verify that
\begin{align*}
\sum_{x \in \mathcal{X}\setminus z} v(x \mid \{z\}) &= \sum_{x \in \mathcal{X}\setminus z} [u_{xz} - \frac{1}{n-1} \sum_{y \in \mathcal{X} \setminus z} u_{yz}] = 0.
\end{align*}
Thus, the assignment is feasible for any $u$ satisfying its sum constraint. Substituting the assignments into the expression for the probability, we have,
\begin{align*}
P(x \mid C) &= \frac{\exp(-\frac{1}{n-1}\sum_{w \in \mathcal{X} \setminus x} u_{wx} + \sum_{z \in C \setminus x} [u_{xz} - \frac{1}{n-1} \sum_{w \in \mathcal{X} \setminus z} u_{wz}])}{\sum_{y \in C} \exp(-\frac{1}{n-1}\sum_{w \in \mathcal{X} \setminus y} u_{wy} + \sum_{z \in C \setminus y} [u_{yz} - \frac{1}{n-1} \sum_{w \in \mathcal{X} \setminus z} u_{wz}])}\\
&= \frac{\exp(-\sum_{z \in C}\frac{1}{n-1}\sum_{w \in \mathcal{X} \setminus z} u_{wz} + \sum_{z \in C \setminus x} u_{xz})}{\sum_{y \in C} \exp(-\sum_{z \in C}\frac{1}{n-1}\sum_{w \in \mathcal{X} \setminus z} u_{wz} + \sum_{z \in C \setminus y} u_{yz})}\\
&=\frac{\exp(\sum_{z \in C \setminus x} u_{xz})}{\sum_{y \in C} \exp(\sum_{z \in C \setminus y} u_{yz} )}
\end{align*}
where the third step follows from $\sum_{z \in C} v(z)$ terms cancelling out across the numerator and denominator. Thus, every $u$ that satisfies the constraint $\sum_{x \in \mathcal{X}} \sum_{y \in \mathcal{X} \setminus x} u_{xy} = 0$ always satisfies the constraints on $v(x)$ and $v(x \mid \{z\})$, and hence the new $P(x \mid C)$ is a valid reparameterization. 
\subsection{Examples of IIA Violations Handled by CDM}\label{section:examples}
Copying over the example from the main text, consider a choice system on $\mathcal{X} = \lbrace a, b, c \rbrace$ where
\begin{align*}
P(a \mid \mathcal{X}) = 0.8, \ \ 
P(b \mid \mathcal{X}) = 0.1, \ \ 
P(c \mid \mathcal{X}) = 0.1. 
\end{align*}
Assuming IIA implies that we can immediately infer the parameters. Using the notation from model $\mathcal{M}_1$, we have that $v(a) = 1.386$, $v(b) = v(c) = -.693$. These three values sum to zero, as per the constraint. We may then state the three relevant pairwise probabilities using these parameters:
\begin{align*}
P(a \mid \{a,b\}) = 0.89, \ \ 
P(b \mid \{b,c\}) = 0.50, \ \ 
P(c \mid \{a,c\}) = 0.11 
\end{align*}
Thus, IIA is full specified and constrained this way. This is in contrast to the CDM, which can specify any arbitrary pairwise probability. As an example, we can model an extreme preference reversal as follows:
\begin{align*}
P(a \mid \{a,b\}) = 0.11, \ \ 
P(b \mid \{b,c\}) = 0.50, \ \ 
P(c \mid \{a,c\}) = 0.89 
\end{align*}
Although $b$ is disproportionately preferred over $a$ in the pair setting, the story almost reverses in the triplet setting. The CDM parameters corresponding to this example are:
$[u_{ab},u_{ac},u_{ba},u_{bc},u_{ca},u_{cb}] = [.693, .693, 2.784, -3.477, 2.784, -3.477]$, where the sum to 0 constraint is being enforced. This notion of preference reversal, and CDM's ability to accommodate it, is actually fairly versatile. Indeed, many of the storied effects in discrete choice, such as those of Similarity Aversion, Asymmetric Dominance, and the Compromise Effect are simply instances of preference reversal. We illustrate this using the following table, adapted from \cite{srivastava2012rational}. $P_{x,A}$ is used to denote the probability of choosing an item $x$ from a set A.
\begin{table}[H]
	\caption{An Overview of the Various Effects} 
	\label{tab:effects}
	\small 
	\centering 
	\begin{tabular}{ccc} 
		\toprule[\heavyrulewidth]\toprule[\heavyrulewidth]
		\textbf{Name} & \textbf{Effect} & \textbf{Constraints} \\
		\midrule
		Preference Reversal & $P_{x,\{x,y\}} > P_{y,\{x,y\}}, \textit{ but } P_{x,\{x,y,z\}} < P_{y,\{x,y,z\}}$ & None\\
		Similarity Aversion & $P_{x,\{x,y\}} > P_{y,\{x,y\}},\textit{ but } P_{x,\{x,y,z\}} < P_{y,\{x,y,z\}}$ & $z \approx x$, splits share\\
		Compromise Effect & $P_{x,\{x,y\}} > P_{y,\{x,y\}},\textit{ but } P_{x,\{x,y,z\}} < P_{y,\{x,y,z\}}$ & $x > y$, $x > z$, $y > z$ \\
		Asymmetric Dominance & $P_{x,\{x,y\}} > P_{y,\{x,y\}},\textit{ but } P_{x,\{x,y,z\}} < P_{y,\{x,y,z\}}$ & $x \approx y$, $y \geq z$ \\
		\bottomrule[\heavyrulewidth] 
	\end{tabular}
\end{table}
Table \ref{tab:effects} provides an overview of the idea that the famous observations of IIA violations in discrete choices are simply instances of preference reversals. Since the CDM can help model such reversals, it can consequently model these effects. 
\subsection{Identifiability and Regularization}\label{section:regl}
In this section, we further explore the concepts developed in the main text about identifiability and regularization. Intricate conditions of identifiability are not unique to the CDM, but are rather widespread in the embeddings literature. These conditions, however, are not very well described or stated anywhere, and especially matter in the embedding setting because regularization is often omitted. Here, we explore a few different models, starting first with the Blade Chest model.
\subsubsection{Blade Chest}
As stated before, we may treat the Blade Chest model as the CDM applied only to the pairwise comparisons. But Theorem \ref{thm:single_set_bad} demonstrates that the CDM is not identified in this setting, hence, neither is the Blade Chest Model. We make this clear as follows. Consider first the full rank case, $d=n$. If $\hat{U}$ is a solution to the problem, then $\tilde{U} = \hat{U} + A$ for any symmetric matrix $A$. Using this, we can consider $d<n$. A subset of solutions when $d<n$ is $\hat{T} + X$, $\hat{C} + Y$, where $X= \beta \hat{C} + \gamma_1 \alpha \beta \hat{T}$, and $Y = \alpha \hat{T} + \gamma_2 \alpha \beta \hat{C}$ where $\alpha,\beta \in \mathbb{R}$, $\gamma_1, \gamma_2 \in \{0,1\}, \gamma_1 \neq \gamma_2$.

 We note that this, however, is only an illustrative small subset to a more general set of solutions that could be better explored through heuristic approaches to the computationally hard affine rank minimization problem.
\subsubsection{Shopper}
Yet another model that suffers from identifiability issues is the Shopper model \cite{ruiz2017shopper}. We refer the reader to the orignal work for a review on the model in order to keep the dicussion here terse.
Consider first the full rank case, $d=n$. If $\hat{U}$ is a solution to the problem, then $\tilde{U} = \hat{U} + \mathbf{1} z^T + \text{diag}(a)$ for any vectors $z,a \in \mathbb{R}^n$. A subset of solutions when $d<n$ is $\hat{T} + x \mathbf{1}^T$, or the origin of the target vector. Though mere shifts of the origin might seem trivial in visualizing the underlying embeddings, these shifts become significant under a measure like cosine distance, or the embeddings use in any absolute, as opposed to relative setting.
\subsubsection{Continuous Bag of Words (CBOW)}
Here, we describe the original CBOW, not the version with negative sampling that is an entirely different objective \cite{rudolph2016exponential}. Consider first the full rank case, $d=n$. If $\hat{U}$ is a solution to the problem, then $\tilde{U} = \hat{U} + \mathbf{1} z^T$ for any vector $z \in \mathbb{R}^n$. A subset of solutions when $d<n$ is $\hat{T} + x \mathbf{1}^T$, or the origin of the target vector. Yet again, when the underlying measure of comparing word similarity is cosine distance---which it frequently is in natural language processing---an origin discrepancy make a difference in underlying task performance.
\subsubsection{Regularization}
A clean solution to issues of uniqueness is to add regularization. Specifically, any amount of $\ell_2$ regularization immediately guarantees identifiability, whereas the same cannot be said of $\ell_1$ regularization. We consider the impact of regularization on the CDM in two specific instances.

\textbf{$\ell_1$ regularization on exponentiated variables.}
Because the CDM is shift invariant, we may set the shift such that the sum of the exponentiated sum of all the rows may be set to 1. That is, $\sum_{y \in \mathcal{X}} \exp(\sum_{x \in \mathcal{X} \setminus y} u_{xy}) = 1$. With such a shift, applying $\ell_1$ regularization to the exponentiated entries may be reformulated as adding a uniform prior of choices from the Universe. Such an idea is described in \cite{ragain2018improving} for the MNL model. This regularization is a valuable addition when the set of observations is small or the comparison graph is irregular. In these settings, the regularization plays a balancing role that is also interpretable for any dataset: additional choices from the universe. However, we know that such an addition alone will not uniquely identify the CDM - especially if the dataset only contains pairwise comparisons, where the CDM will not be identified even with an arbitrarily large sample size. Even with datasets of a choice set size greater than 2, the dataset still requires samples from a diverse range of choice sets within that size before it is identifiable with the regularization. This is consistent with the view that $\ell_1$ does not always identify the CDM.

\textbf{$\ell_2$ regularization on the $U$ matrix.} As stated earlier, any small amount of $\ell_2$ regularization immediately identifies the CDM. Since the ``pairwise comparisons only'' setting suffers in a rather extreme way from identifiability issues, understanding the role $\ell_2$ regularization plays there is important. We recall from earlier than in the setting of pairwise comparisons, the CDM matrix $U$ is only specified up to a symmetric matrix $A$ when inferred from pairwise comparisons. Since $\ell_2$ regularization will minimize the entrywise norm of the $U$ matrix, $A$ will be chosen to be zero. That is, the $U$ matrix will be antisymmetric. We may then use this property to solve for parameter $u_{xy}$ as a function of the pairwise probabilities:
\begin{align*}
u_{xy} = \frac{1}{2}\log\Big(\frac{P_{x,\{x,y\}}}{P_{y,\{x,y\}}}\Big) 
\end{align*}
It is most interesting to look at 
\begin{align*}
u_{xz}-u_{yz} = \frac{1}{2}\log\Big(\frac{P_{z,\{y,z\}}P_{x,\{x,z\}}}{P_{y,\{y,z\}}P_{z,\{x,z\}}}\Big).
\end{align*}
Since $u_{xz}-u_{yz}$ corresponds to the influence a third item $z$'s presence has on the choice between $x$ and $y$, it is interesting that the relative intransitivities of the three items in their respective pairwise settings are leveraged to describe this influence in the triplet case. This is quite possibly the best outcome one could hope for having just pairwise comparisons, and demonstrates the value of regularization. 
\subsection{Auxiliary Lemmas}
\begin{lemma}\label{lemma:norm_results}
For $\Sigma_{\mathcal{D}} \defeq \frac{1}{m^2} X(\mathcal{D})L^\dagger X(\mathcal{D})^T$, where the remaining quantities are defined in the proof of Theorem \ref{thm:CDM_bound}, we have,
\begin{align*}
\textbf{tr}(\Sigma_{\mathcal{D}}) = \frac{d-1}{m} && \textbf{tr}(\Sigma_{\mathcal{D}}^2) = \frac{(d-1)^2}{m^2} && ||\Sigma_{\mathcal{D}}||_\text{op} = \frac{1}{m}.
\end{align*}
\end{lemma}
{\bf Proof.}
Consider first that $L = \frac{1}{m}X(\mathcal{D})^TX(\mathcal{D})$. Since $L$ is symmetric and positive semidefinite, it has an eigenvalue decomposition of $U \Lambda  U^T$. By definition, the Moore-Penrose inverse is $L^\dagger = U \Lambda ^\dagger U^T$. We must have that $X(\mathcal{D}) = \sqrt{m}V \Lambda ^{\frac{1}{2}}U^T$ for some orthogonal matrix $V$ in order for $L$ to equal $\frac{1}{m}X(\mathcal{D})^TX(\mathcal{D})$. With these facts, we have 
\begin{align*}
\frac{1}{m^2} X(\mathcal{D})L^\dagger X(\mathcal{D})^T &=\frac{1}{m^2}\sqrt{m}V \Lambda ^{\frac{1}{2}}U^TU \Lambda ^\dagger U^TU\Lambda ^{\frac{1}{2}}V^T \sqrt{m}\\
&=\frac{1}{m}V\Lambda  \Lambda ^\dagger V^T.
\end{align*}
That is, $\Sigma_{\mathcal{D}}$ is a positive semi-definite matrix with spectra corresponding to $d-1$ values equaling $\frac{1}{m}$, and the last equaling 0. The three results about the traces and the operator norm immediately follow.
\end{document}